\documentclass{article}

\PassOptionsToPackage{numbers, compress}{natbib}

\usepackage[preprint]{neurips_2021}

\usepackage[utf8]{inputenc} 
\usepackage[T1]{fontenc}    
\usepackage{nicefrac}       
\usepackage{microtype}      
\usepackage{xcolor}         

\usepackage{soul}
\usepackage{url}
\usepackage[small]{caption}
\usepackage{subcaption}
\usepackage{graphicx}
\usepackage{amsmath}
\usepackage{amsthm}
\usepackage{amssymb}
\usepackage{amsfonts}
\usepackage{dsfont}         
\usepackage{bm}
\usepackage{booktabs}
\usepackage{multirow}
\usepackage{enumerate}
\usepackage{lscape}
\usepackage{graphicx}
\usepackage[
    vlined,
    ruled,
    linesnumbered,
]{algorithm2e}
\usepackage[hidelinks=true]{hyperref}
\usepackage[
    noabbrev,
    capitalise,
    nameinlink
]{cleveref}
\crefname{figure}{Fig.}{Figs.}
\crefname{equation}{Eq.}{Eqs.}
\usepackage{pdfpages}

\crefname{proposition}{Proposition}{Propositions}

\newtheorem{assume}{Assumption}
\crefname{assume}{Assumption}{Assumptions}

\newtheorem{lemma}{Lemma}
\crefname{lemma}{Lemma}{Lemmas}

\newtheorem{thm}{Theorem}
\crefname{thm}{Theorem}{Theorems}

\def\blinded{}
\newcommand{\blind}[2]{\ifdefined\blinded#2\else#1\fi}

\newcommand{\E}{\ensuremath{\mathbb{E}}}    
\newcommand{\Tau}{\mathcal{T}}              

\newcommand{\abs}[1]{\lvert#1\rvert}
\newcommand{\norm}[1]{\lVert#1\rVert}

\DeclareMathOperator*{\argmax}{arg\,max}    
\DeclareMathOperator*{\argmin}{arg\,min}    

\newcommand{\States}{\ensuremath{\mathcal{S}}}
\newcommand{\Actions}{\ensuremath{\mathcal{A}}}

\newcommand{\Data}{\ensuremath{\mathcal{D}}}

\newcommand{\ie}{\textit{i.e.}}             
\newcommand{\eg}{\textit{e.g.}}             
\newcommand{\apriori}{\textit{a priori}}    

\newcommand{\wrt}{\text{w.r.t.}}            

\graphicspath{{figures/}}

\renewcommand{\hl}[1]{{#1}}

\newcommand{\intermargin}{d}
\newcommand{\intramargin}{\gamma}

\let\oldparagraph\paragraph
\def\paragraph{\vspace{-1.0em}\oldparagraph}

\title{LiMIIRL: Lightweight Multiple-Intent Inverse Reinforcement Learning}

\author{%
    Aaron J.~Snoswell \\
    School of Information Technology and Electrical Engineering \\
    The University of Queensland \\
    \href{mailto:a.snoswell@uq.edu.au}{a.snoswell@uq.edu.au} \\
    \And
    Surya P.~N.~Singh \\
    Intuitive Surgical \\
    \href{mailto:surya.singh@intusurg.com}{surya.singh@intusurg.com} \\
    \And
    Nan Ye \\
    School of Mathematics and Physics \\
    The University of Queensland \\
    \href{mailto:nan.ye@uq.edu.au}{nan.ye@uq.edu.au} \\
}

\begin{document}

\maketitle

\begin{abstract}
    Multiple-Intent Inverse Reinforcement Learning (MI-IRL) seeks to find a reward function ensemble to rationalize demonstrations of different but unlabelled intents. 
    
    \indent Within the popular expectation maximization (EM) framework for learning probabilistic MI-IRL models, we present a warm-start strategy based on up-front clustering of the demonstrations in feature space.
    Our theoretical analysis shows that this warm-start solution produces a near-optimal reward ensemble, provided the behavior modes satisfy mild separation conditions.
    We also propose a MI-IRL performance metric that generalizes the popular Expected Value Difference measure to directly assesses learned rewards against the ground-truth reward ensemble.
    Our metric elegantly addresses the difficulty of pairing up learned and ground truth rewards via a min-cost flow formulation, and is efficiently computable.
    
    \indent We also develop a MI-IRL benchmark problem that allows for more comprehensive algorithmic evaluations.
    On this problem, we find our MI-IRL warm-start strategy helps avoid poor quality local minima reward ensembles, resulting in a significant improvement in behavior clustering.
    Our extensive sensitivity analysis demonstrates that the quality of the learned reward ensembles is improved under various settings, including cases where our theoretical assumptions do not necessarily hold.
    Finally, we demonstrate the effectiveness of our methods by discovering distinct driving styles in a large real-world dataset of driver GPS trajectories.
\end{abstract}

\section{Introduction}
\label{sec:introduction}

Inverse Reinforcement Learning (IRL) methods search for a reward function to rationalize demonstrated behaviour \cite{Russell1998}.
However, in many cases, these demonstrations consist of different but unobserved behavior modes.
For example, taxi drivers may choose different routes even when the pickup location and the destination are the same, upon the passenger's request to minimize the cost or the time (\cref{fig:porto-multimodal}, left).
Main approaches to deal with such multiple (unobserved) intent IRL (MI-IRL) problems are Expectation Maximization (EM) \cite{Babes-Vroman2011,Ramponi2020} or non-parametric methods \cite{Choi2012,Almingol2015}.
We focus on the EM formulation, which requires the number of intents $K$ as a hyper-parameter, but is computationally more practical.
Unfortunately, EM algorithms easily get stuck in poor local minima, and thus good initialization is often a subject of algorithmic design and theoretical analysis \cite{Vlassis2002,Balakrishnan2017}.

For MI-IRL, we similarly observed that the EM algorithm is prone to local minima that may have good likelihood, but fail to learn good reward functions for the demonstrated behaviors.
To address this, we propose Lightweight Multiple-Intent IRL (LiMIIRL), a simple warm-start strategy that constructs an initial solution by clustering the demonstrations in feature space, and then uses EM to improve the solution.
Theoretically, we bound the error of the clustering-based initial solution,
which shows that the initial solution can be near-optimal when the behaviors are
approximately separable in the feature space. 
To the best of our knowledge, this is the first provable performance guarantee for an EM initialization in the MI-IRL context.
Empirically, our experiments show that LiMIIRL helps avoid poor quality solutions, 
and allows terminating the EM loop after a very small number of iterations, significantly speeding up the convergence of the EM algorithm as compared to random initializations.
These results show that the simple strategy of clustering in the feature space can effectively help grouping demonstrations into clusters with distinct reward objectives.

Another contribution of this paper is a performance metric for MI-IRL
algorithms (\cref{par:min-cost-flow-metric}). 
Our metric allows direct comparison (without reference to a dataset) between two sets of multiple-intent behaviors, each represented by an ensemble of reward functions and is efficiently computable by solving a minimum cost flow problem.

In the remainder of this paper, we present LiMIIRL with a theoretical analysis (\cref{sec:LiMIIRL}) and introduce a new MI-IRL metric (\cref{par:min-cost-flow-metric}).
We demonstrate the effectiveness of LiMIIRL and analyse its sensitivity 
\wrt{} problem and algorithm hyper-parameters on a new MI-IRL benchmark,
which allows for more comprehensive testing of MI-IRL algorithms
(\cref{sec:element-world}). 
LiMIIRL is also shown to effectively discover distinct driving behaviours in a 
real-world dataset of taxi driving trajectories (\cref{sec:porto-problem}). We elaborate context in both the background
 and related work (\cref{sec:background,sec:related-work}).
Our code is available at \blind{\url{https://github.com/aaronsnoswell/multimodal-irl}}{[URL removed for blind review]}.

\section{Background}
\label{sec:background}

\begin{figure}[t]
    \centering
    \begin{subfigure}{.3\columnwidth}
        \centering
        \includegraphics[width=\linewidth]{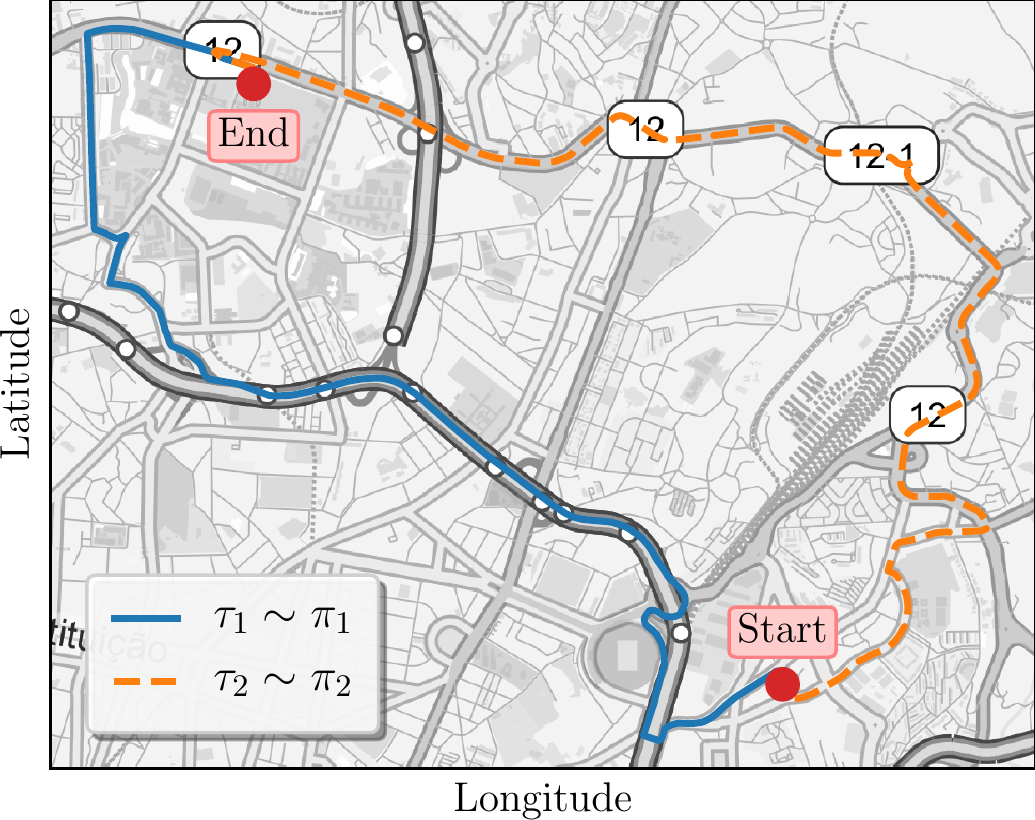}
    \end{subfigure}%
    \hfill
    \begin{subfigure}{.3\columnwidth}
        \centering
        \includegraphics[width=\linewidth]{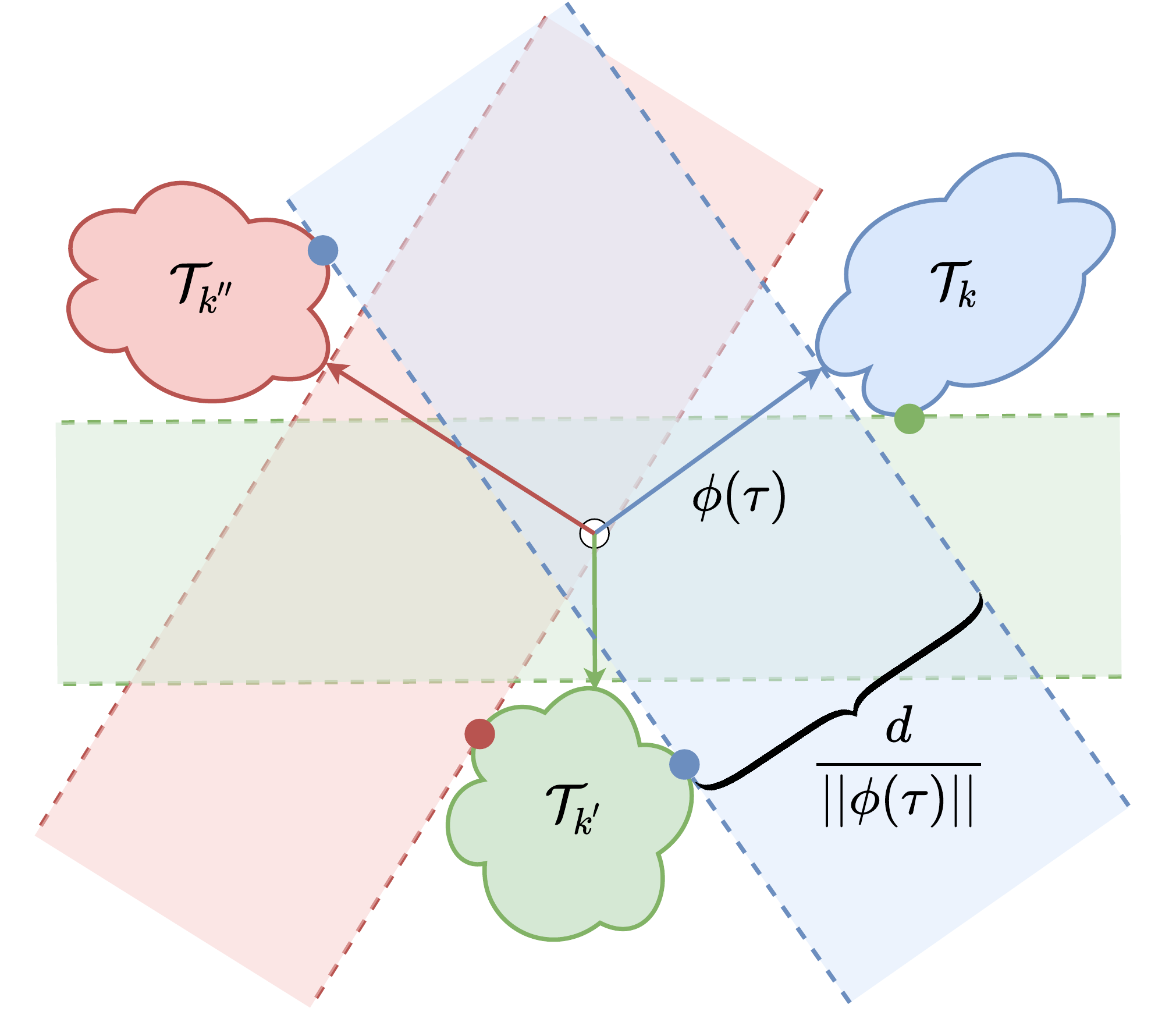}
    \end{subfigure}%
    \hfill
    \begin{subfigure}{0.3\columnwidth}
        \centering
        \includegraphics[width=\linewidth]{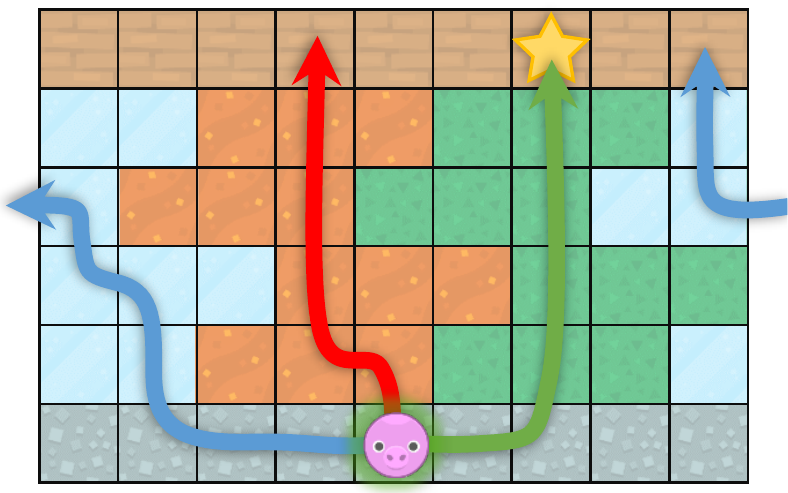}
    \end{subfigure}%
    \vspace*{-5pt}
    \caption{%
        (Left): Multiple behaviour intents in the \emph{Porto} driver behaviour dataset.
        Taxi drivers going from the same pick-up location to the same destination may choose very different routes.
        \hl{(Middle): An illustration of \mbox{\cref{assume:clusters}}} for
				$K=3, |\phi| = 2$, when $\zeta = 0$.
        Demonstrations not in cluster $\Tau_k$ are shown to be at least some
				margin away from $\Tau_k$ when projected onto $\tau \in \Tau_k$.
				When $\zeta > 0$, some trajectories in $\Tau_{k}$ are allowed to be
				closer to cluster $k'$ than cluster $k$ when projected on $\tau$.
        (Right): Our \textit{ElementWorld} MI-IRL benchmark is a gridworld problem that wraps around the x-axis.
        The agent starts in the bottom-most row of the grid world and must navigate to the top.
        While en-route, the agent prefers to visit states matching it's desired element. 
    }
    \label{fig:porto-multimodal}
    \label{fig:assumption-a}
    \label{fig:element-world}
    \vspace{-10pt}
\end{figure}

We consider behaviours generated by one or more experts acting in some environment.
We model the environment as a partial Markov Decision Process (MDP) $\mathcal{M} = \langle \States, \Actions, p_{0}, T, \gamma \rangle$ with state space $\States$ and action space $\Actions$, discount factor $\gamma \in [0, 1)$, starting state probabilities $p_0(s)$ and a transition function $T(s, a, s') = p(s' \mid s, a)$.
In MI-IRL, we are given a dataset of demonstration trajectories $\Data = \{\tau_{1}, \ldots, \tau_{N}\}$ collected when an ensemble of $K$ experts interact with the environment to optimize their individual reward functions $R_k(s, a, s') \triangleq \theta_k^\top \phi(s, a, s')$.
Formally, we assume each trajectory $\tau$ is sampled from $p(\tau \mid \rho_{1:K}, \pi_{1:K}) = \sum_{k=1} \rho_{k} ~ \pi_{k}(\tau)$, where $\{\rho_i\}_{i=1}^K$ are positive mixture weights, and each expert $\pi_{i}$ is represented as a distribution on $\Tau$, the set of all trajectories of length at most $L$.
We work in the unsupervised setting, where the trajectories are not labeled with the corresponding behavior mode.

Given a transition-based feature function $\phi: \States \times \Actions \times \States \to \mathbb{R}$, we define its extension to trajectories by $\phi(\tau) = \sum_{t=1}^{|\tau|-1} \gamma^{t-1} \phi(s_{t}, a_{t}, s_{t+1})$, where $\tau = (s_{1}, a_{1}, s_{2}, \ldots, s_{|\tau|})$.

We use $\Data$ to learn an ensemble of models $p(\tau \mid \rho_{1:K}, \theta_{1:K}) = \sum_{k=1}^{K} \rho_{k} ~ p(\tau \mid \theta_{k})$, where $\{\rho_{1:K}, \theta_{1:K}\}$ are the parameters with each $\rho_{i} \ge 0$ and $\sum_{i} \rho_{i} = 1$, and $p(\tau \mid \theta_{k})$ is a distribution on $\Tau$ with parameters $\theta_{k}$.
The parameters $\{\rho_{1:K}, \theta_{1:K}\}$ are 
chosen to maximize the likelihood $p(\Data \mid \rho_{1:K}, \theta_{1:K})$.
The EM algorithm, first applied to MI-IRL by \cite{Babes-Vroman2011}, starts from some initial parameter estimates $\{\rho_{1:K}^{(0)}, \theta_{1:K}^{(0)}\}$, and then iteratively improves the parameter estimates until convergence. 

At iteration $t$, EM updates the parameters as follows
\vspace*{-5pt}
\begin{align*}
	u_{ik}^{(t+1)}
	&= \frac{
	    \rho_{k}^{(t)} p(\tau_{i} \mid \theta_{k}^{(t)})
	}{
	    \sum_{k'} \rho_{k'}^{(t)} p(\tau_{i} \mid \theta_{k'}^{(t)})
    },
    &
	\rho_{k}^{(t+1)} 
	&= \frac{1}{N} \sum_{i} u_{ik}^{(t+1)},
	&
	\theta_{k}^{(t+1)} = \argmax_{\theta} \sum_{i} u_{ik}^{(t+1)} p(\tau_{i} \mid \theta).
\end{align*}
%
%
We say that a solution $\Theta_{t} = \{\rho_{1:K}^{(t)}, \theta_{1:K}^{(t)}\}$ is $\epsilon$-optimal if 
$\frac{1}{|\Data|}\sum_{i} \sum_{k} |u_{ik}^{(t+1)} - u_{ik}^{(t)}| < \epsilon$.
Intuitively, a small $\epsilon$ indicates that the EM algorithm has more or less
converged to a local optimum.

\section{The LiMIIRL Algorithm}
\label{sec:LiMIIRL}

Our LiMIIRL algorithm provides a simple way to generate good initial parameter estimates $\Theta^{(0)} = (\rho^{(0)}_1, \dots, \rho^{(0)}_{K}, \theta^{(0)}_1, \dots, \theta^{(0)}_{K})$ by clustering demonstrations in feature space.
The initial solution is generated in three steps.

First, LiMIIRL runs a user-specified clustering algorithm on the feature vectors of the training trajectories to obtain a partition $\{C_{1}, \ldots, C_{K}\}$ of $\Data$.
Second, LiMIIRL estimates $\rho_{1:K}$ using $\rho_{k}^{(0)} = \sum_{i=1}^N u_{ik}^{(0)} / N$, where the membership variable $u_{ik}^{(0)} = \mathbb{I}(i \in C_{k})$ denotes whether trajectory $i$ belongs to cluster $k$.
Finally, LiMIIRL estimates $\theta_{1:K}^{(0)}$.
We consider two ways to do this, \textit{mean} initialization, or \textit{Maximum Likelihood Estimate (MLE)} initialization, given respectively by,
\begin{align}
	\theta_{k,\text{Mean}}^{(0)} &= \Phi_{k} \triangleq \frac{1}{|C_{k}|} \sum_{\tau \in C_{k}} \phi(\tau),
	&
	\theta_{k,\text{MLE}}^{(0)} &= \argmax_{\theta} \sum_{\tau \in C_{k}} \ln p(\tau \mid \theta).
\end{align}
In our algorithm, $p(\tau \mid \theta)$ can be any single-intent IRL model, and we experimented with the MaxEnt IRL model \cite{Ziebart2008}, 
Maximum Likelihood IRL \cite{Babes-Vroman2011}, and $\Sigma$-Gradient IRL \cite{Ramponi2020}.

Below, we provide a theoretical justification of why clustering in feature space is an effective choice for initializing an ensemble of behaviours based on the MaxEnt IRL model, which has exact and computationally efficient solution methods available \cite{Snoswell2020}. 
Specifically, this entails choosing 
$p(\tau \mid \theta) = q(\tau) e^{\phi(\tau)^{\top} \theta} / Z(\theta)$, where 
$q(\tau) = p_0(s_{1}) \prod_{t=1}^{|\tau|-1} p(s_{t+1} \mid s_{t}, a_{t})$, and 
$Z(\theta) = \sum_{\tau' \in \Tau} q(\tau') e^{\theta^\top \phi(\tau')}$ is the
partition function.

Interestingly, our empirical results suggest that this warm-start technique also provides a good initial reward ensemble under alternate behaviour models, such as the Boltzmann model in ML-IRL and BIRL \cite{Babes-Vroman2011,Ramachandran2007} and the gradient-based model in $\Sigma$-GIRL \cite{Ramponi2020}.

\subsection{Theoretical Analysis}
\label{subsec:theoretical-analysis}

To be able to discern different behavior intents in the demonstration data, the
intents must be distinguishable.
When representing the demonstrations as feature vectors, this amounts to
requiring the feature vectors for \emph{typical/characteristic} trajectories
from each expert to form clusters that are mostly distinct from each other, and
the \emph{atypical/ambiguous} trajectories can be different but not too
different from the typical trajectories.
Specifically, to formulate this intuition, we say that a trajectory
$\tau$ \emph{$\gamma$-separates} two sets of trajectories $\Tau_{a}$ and
$\Tau_{b}$ (in the feature space) if 
$\phi(\tau)^{\top}(\phi(\tau_{a}) - \phi(\tau_{b})) \ge \gamma$ for all
$\tau_{a} \in \Tau_{a}$ and $\tau \in \Tau_{b}$.
Geometrically, when $\gamma > 0$, the trajectories in $\Tau_{a}$ and $\Tau_{b}$
are separated by a margin of at least $\gamma/\norm{\phi(\tau)}$ when projected
on $\tau$ in the feature space.
That is, we can find a SVM with a margin of at least
$\gamma/\norm{\phi(\tau)}$ to separate trajectories in $\Tau_{a}$ and
$\Tau_{b}$.
When $\gamma \le 0$, the projections of $\Tau_{a}$ and $\Tau_{b}$ on $\tau$ may 
overlap over an interval of length $|\gamma|/\norm{\phi(\tau)}$.
\begin{assume}[Approximate cluster separation] \label{assume:clusters}
	Let $\Tau_{k}$ be the set of trajectories that can be generated by expert
	$\pi_{k}$, and $\Tau_{\sim k} = \Tau - \Tau_{k}$.
	We assume that there exists $\zeta \in [0, 1/2)$ such that each $\Tau_{k}$ can be partitioned into two sets: 
	a set of typical trajectories $\Tau_{k}^{+}$ and a set of atypical trajectories $\Tau_{k}^{-}$ such that 
	$Q_{k}^{+} = \sum_{\tau \in \Tau_{k}^{+}} q(\tau) 
	\ge 
	(1 - \zeta) Q_{k}$, or equivalently, 
	$Q_{k}^{-} 
	= \sum_{\tau \in \Tau_{k}^{-}} q(\tau) 
	\le 
	\zeta Q_{k}$,
	where $Q_{k} = \sum_{\tau \in \Tau_{k}} q(\tau)$.
	In addition, the typical trajectories and atypical trajectories satisfy the
	following properties
	\begin{enumerate}[(a)]
		\item There exist $\intermargin, \intramargin > 0$, such that for any $k \in [K]$, 
			\begin{enumerate}[(i)]
				\item each $\tau \in \Tau_{k}^{+}$ $d$-separates $\Tau_{k}^{+}$ and
					$\Tau_{\sim k}$, and $\intramargin$-separates $\Tau_{k}^{+}$ and $\Tau_{k}^{-}$. 
				\item each $\tau \in \Tau_{k}^{-}$ $(-\intermargin)$-separates $\Tau_{k}^{+}$ and
					$\Tau_{\sim k}$, and $(-\intramargin)$-separates $\Tau_{k}^{+}$ and $\Tau_{k}^{-}$.
			\end{enumerate}
		\item $\forall k \in [K], ~\exists~ 
			r_{k} < \min_{\tau \in \Tau_{k}^{+}} \intermargin/2\norm{\phi(\tau)}$
			s.t.
			$\|\phi(\tau) - \phi(\tau')\| \le 2r_{k}$ for all $\tau, \tau' \in \Tau_{k}^{+}$.
			That is, the trajectories in $\Tau_{k}^{+}$ are contained in a ball of radius $r_{k}$ in the feature space.
	\end{enumerate}
\end{assume}
Intuitively, we assume that a typical trajectory in a cluster can be used to
distinguish typical trajectories from atypical trajectories and trajectories
in other clusters with margins $\intramargin$ and $\intermargin$ respectively.
On the other hand, an atypical trajectory can fail to do so, but the overlaps
are upper bounded by $\intramargin$ and $\intermargin$ respectively, so that
they cannot be too different from the typical trajectories --- 
our analysis can be easily extended to alternative choices of the upper bounds,
but we have chosen $\intramargin$ and $\intermargin$ for simplicity.
\Cref{fig:assumption-a} (middle) provides an illustration of \Cref{assume:clusters} for the case when $\zeta = 0$.
In practice, a useful trick is to shift the features so that the mean feature
vector is 0.
This does not change the parameters of the optimal MaxEnt IRL model, but allows
us to exploit the clusters in the data via \cref{assume:clusters}.


With mostly-distinct clusters, we can reasonably expect a clustering algorithm
to be able to discover the underlying clusters in the dataset with high
accuracy.
\begin{assume}[Approximately correct clustering] \label{assume:kmeans}
	There exists $\delta \in [0,1/2)$, such that at least $1-\delta$ of the trajectories
	in $C_{k}$ are in $\Tau_{k}^{+}$.
	In addition, the other trajectories 
	$(-\intermargin)$-separate $\Tau_{k}^{+}$ and $\Tau_{k'}$, and 
	$(-\intramargin)$-separate $\Tau_{k}^{+}$ and $\Tau_{k}^{-}$.
\end{assume}
Intuitively, most the trajectories in $\Data$ are typical, and the clustering
algorithm is able to cluster most of them correctly.
However, the algorithm may fail on some trajectories, and these
trajectories are assumed to be not worse than the atypical trajectories in terms
of their separation ability.

To quantify the quality of the mean initialization, we first observe that
$\Phi_{k}$ is like a typical trajectory in $\Tau_{k}$ (all proofs are in the
supplementary file).
\begin{lemma} \label{lem:phi} 
	For any $k \in [K]$,
	$\Phi_{k}$ $\tilde{\intermargin}$-separates $\Tau_{k}^{+}$ and $\Tau_{\sim k}$
	for 
	$\tilde{\intermargin} 
	= 
	(1 - 2 \delta) \intermargin$.
	In addition, $\Phi_{k}$
	$\tilde{\intramargin}$-separates 
	$\Tau_{k}^{+}$ and $\Tau_{k}^{-}$ 
	for $\tilde{\intramargin} = (1 - 2\delta) \intramargin$.
\end{lemma}
Our second observation is that $p(\tau \mid \Phi_{k})$ is large for typical
trajectories in $\Tau_{k}$, and small for trajectories not in $\Tau_{k}$.
\begin{lemma} \label{lem:lu}
	For any $\tau \in \Tau_{k}^{+}$, we have 
	$p(\tau \mid \Phi_{k}) 
	\ge 
	\frac{(1-\zeta) q(\tau) e^{\tilde{\intermargin}-2r_{k} \norm{\Phi_{k}}}}
	{(1-\zeta) Q_{k} e^{\tilde{\intermargin}} + Q_{\sim k}}
	\frac
	{(1-\zeta) e^{\tilde{\intramargin}}}
	{\zeta + (1-\zeta) e^{\tilde{\intramargin}}}$,
	where $Q_{\sim k} = \sum_{\tau \notin \Tau_{k}} Q_{k'}$.
	In addition, for any $\tau \notin \Tau_{k}$, we have 
	$p(\tau \mid \Phi_{k}) 
	\le 
	\frac{q(\tau)}{(1-\zeta) Q_{k} e^{\tilde{\intermargin}} + q(\tau)}$.
\end{lemma}

With the above two observations, we can then quantify the quality of the mean
initialization.

\begin{thm} \label{thm:near-optimality}
	The mean initialisation LiMIIRL initial estimates $\{\rho_{1:K}^{(0)},
	\theta_{1:K}^{(0)}\}$ are $\epsilon$-optimal for 
	$\epsilon = 2\delta + (1-\delta) \frac{2 (K-1)}{\beta + K-1}$, where 
	$\beta 
	= 
	\min_{k \neq k'} 
	\frac
	{(1-\zeta)^{2} e^{\tilde{\intramargin}}}
	{\zeta + (1-\zeta) e^{\tilde{\intramargin}}}
	\frac{|C_{k}| ((1-\zeta) Q_{k'} e^{\tilde{\intermargin}})}
	{|C_{k'}| ((1-\zeta) Q_{k} e^{\tilde{\intermargin}} +Q_{\sim k})}
	e^{\tilde{\intermargin} - 2r_{k} \norm{\Phi_{k}}}$.
	%
\end{thm}
The bound in \Cref{thm:near-optimality} implies that the quality of the
warm-start solution improves when the separability of the behaviors and the
quality of the clustering algorithm are higher.
Specifically, noting that 
$\tilde{\intermargin} 
= (1 - 2\delta) \intermargin$ and 
$\tilde{\intramargin} 
= (1 - 2\delta) \intramargin$,
we can see that $\epsilon$ is small 
when $\zeta$ is small (smaller proportion of atypical trajectories),
$\intermargin$ is large (more distinct behaviors), 
$r_{k}$'s are smaller (more concentrated clusters),
$\intramargin$ is large (larger difference between typical and atypical
trajectories),
and $\delta$ is small (more accurate clustering algorithm).
In addition, $\epsilon$ converges to $\delta$ at a rate exponential in $\intermargin$.

For the MLE initialization, obtaining a closed-form error bound is much more involved.
However, we highlight that the above analysis also provides an insight that when the scaled margin is large, the MLE initialization would work well too.
To see this, note that cluster $k$ is roughly contained within a hyper-cone having the origin as the vertex and an opening angle $\alpha = \sin^{-1}(r_{k}/\norm{\Phi_{k}})$.
Any $\theta_{k}$ becomes less likely to be optimal as it forms an angle of more
than $2 \alpha$ with $\Phi_{k}$, as follows: when $\theta_{k}$'s angle with
$\Phi_{k}$ changes from 0 to more than $2 \alpha$, for any $\tau \in \Tau_{k}$,
$e^{\phi(\tau)^{\top} \theta_{k}}$ decreases, because the angle between $\phi(\tau)$ and $\theta_{k}$ changes from at most $\alpha$ to at least $\alpha$.
On the other hand, there is an empty space with margin $d_{k} = d/\norm{\Phi_{k}}$ below the cluster, thus trajectories not in $\Tau_{k}$ initially form large angles with $\theta_{k}$, and have a small contribution to the partition function, then gradually become more important as $\theta_{k}$ moves away from $\Phi_{k}$, eventually making the probability of trajectories in $\Tau_{k}$ small and the likelihood of $\theta_{k}$ small.


\section{MI-IRL Performance Metrics}
\label{sec:mi-irl-performance-metrics}

We introduce our novel metric to evaluate an MI-IRL algorithm's performance in terms of the quality of the learned reward ensemble.
We also evaluate an MI-IRL algorithm's clustering performance and describe the well-known Adjusted max-Normalized Information Distance (ANID) for completeness.

\paragraph{Reward Ensemble Performance.}
\label{par:min-cost-flow-metric}

For single-intent IRL problems the Expected Value Difference (EVD) is an appropriate measure of regret of a learned reward $R^\text{L}$ \wrt{} a true reward function $R^\text{GT}$, $\text{EVD}(R^{\text{GT}}, R^{\text{L}}) \triangleq \E_{s \sim p_{0}}[\bm{v}_{\pi^*_{R^{\text{GT}}}}(s) - \bm{v}_{\pi^*_{R^{\text{L}}}}(s)]$, where $\bm{v}_{\pi}$ is the value function for policy $\pi$, and $\pi^*_{R}$ is the optimal policy \wrt{} reward function $R$ \cite{Choi2011,Levine2011}.
Previous authors applied this metric in the multiple-intent setting by averaging the EVD of the true and learned rewards for each trajectory \cite{Babes-Vroman2011,Choi2012} -- this process does not directly compare the reward ensembles but instead depends on a sample of demonstrations.

We propose a Generalized EVD (GEVD) metric that directly evaluates a learned reward ensemble $\{\rho^{\text{L}}_{1:K}, R^\text{L}_{1:K}\}$ against the ground truth reward ensemble $\{\rho^{\text{GT}}_{1:K'}, R^\text{GT}_{1:K'}\}$.
To motivate, note that a reward ensemble remains the same even if the rewards are reordered, thus for the special case where two reward ensembles are the same, we can perform a bipartite matching to pair up the rewards.

We generalize this idea to handle the harder cases in which $K \neq K'$, and the weights may not be the same.
Essentially, this requires an additional reward split step before pairing up the rewards.
Instead of just finding the optimal pairing as in the simplest case above, we now need to find the best way to split-and-pair rewards.
This can be naturally formulated as follows, $\text{GEVD}(\mathbb{R}^\text{GT}, \mathbb{R}^\text{L}) \triangleq
    \argmin_{
			\{w_{ij}\} : \sum_j w_{ij} = \rho^{\text{GT}}_i, 
			\sum_i w_{ij} = \rho^{\text{L}}_j
    }
    \sum_{i,j}
    w_{ij}\,e_{ij}$
%
%
 where $e_{ij} = \text{EVD}(R^{\text{GT}}_{i}, R^{\text{L}}_{j})$.
It is easy to show that the above problem is equivalent to a min-cost flow problem (see our note in the Supplementary file), thus the GEVD can be efficiently computed using standard graph manipulation libraries.

The range of GEVD is $\left[0, \frac{\max_k \Delta_k}{1 - \gamma}\right]$, where $\Delta_k$ is the difference between the maximum and minimum possible rewards under the $k$-th ground truth reward reward function.
GEVD is 0 iff there is a way to split-and-pair ground truth and learned rewards such that each pair of rewards lie in an equivalence class (in the sense of having the same optimal policies), and depends only on the reward ensembles.
We can normalize GEVD into the range of $[0, 1]$ by dividing it with $\frac{\max_k \Delta_k}{1 - \gamma}$, however a better normalization can be achieved by dividing it with a tighter upper bound $\sum_{i} \rho^{\text{GT}}_{i} e_{i}^{*}$, where $e_{i}^{*} = \E_{s \sim p_{0}}[ \bm{v}_{\pi^*_{R^{\text{GT}}}}(s) - \bm{v}_{\pi^{-}_{R^{\text{GT}}}}(s) ]$, with $\pi^{-}_{R}$ being the policy \textit{minimizing} the reward $R$.
We can calculate $\pi^{-}_{R}$ using the same algorithm for computing $\pi^{*}_{R}$ (\eg{} value iteration), so the normalization only increases the
computation time by a small factor.
We believe the GEVD measure provides an elegant means for evaluating arbitrary reward ensembles in controlled IRL experiments where we have access to the MDP model, and we believe this is the first such metric for the MI-IRL problem that does not require reference to a held-out testing trajectory dataset.
Furthermore, the GEVD actually provides a useful tool to diagnose poorly performing MI-IRL solutions by inspecting reward pairings with large EVDs, naturally guiding the practitioner toward parts of the model that may require attention.

\paragraph{Clustering Performance.}
MI-IRL requires both reward ensemble learning, and demonstration clustering -- both tasks should be evaluated in holistic evaluation of any MI-IRL method.
The Adjusted max-Normalized Information Distance (ANID) is a popular and theoretically justified performance metric for assessing a mixture model's (hard \textit{or} soft) clustering capability \cite{Vinh2010}.
Given a learned MI-IRL mixture model, we can compute its \emph{responsibility matrix} $\bm{U} = \{u_{ik}\}$ where $u_{ik}$ is the probability that demonstration $i$ belongs to cluster $k$.
Similarly, we can compute the responsibility matrix $\bm{V} = \{v_{ik'}\}$ for the ground-truth MI-IRL mixture model.
The ANID is given by $\text{ANID}(\bm{U}, \bm{V}) \triangleq 1 - \left(I(\bm{U},\bm{V}) - \E[I(\bm{U}',\bm{V}')]\right) / \left(\max \{ H(\bm{U}),H(\bm{V}) \} - \E[I(\bm{U}',\bm{V}')]\right)$, where $\bm{U}'$ and $\bm{V}'$ are random responsibility matrices of the same dimensions as $\bm{U}$ and $\bm{V}$ respectively distributed according to an appropriate random clustering model,\footnote{For hard clusters, hypergeometric models are commonly used.
In our soft clustering case each row of $\bm{U}'$ is independently sampled from the Dirichlet distribution $Dir(1/K, \ldots, 1/K)$, and similarly for $\bm{V}'$.} $I(\bm{U}, \bm{V})$ is the mutual information between $k, k'$ following a joint distribution $p(k, k') = \frac{1}{N} \sum_{i=1}^{N} u_{ik} v_{ik'}$, and $H(\bm{U})$ and $H(\bm{V})$ are respectively the entropies of $p(k)$ and $p(k')$, the two marginal distributions of $p(k, k')$.
The ANID has many attractive properties: it is symmetric \wrt{} to $\bm{U}$ and $\bm{V}$, is 0 iff $\bm{U} = \bm{V}$ modulo column re-ordering, is computationally simple to calculate, controls for random chance agreement between clusters when $K$ is large, is stochastically normalized to the range $[0, 1]$, and in the limit $N \rightarrow \infty$ converges to a proper metric \cite{Vinh2010}.

\section{ElementWorld Benchmark Experiment}
\label{sec:element-world}

Previous MI-IRL methods have used a simple two-intent puddle gridworld environment for evaluation \cite{Babes-Vroman2011,Ramponi2020}.
To investigate the scalability of our method on a more substantial problem, we generalized this environment to the case of many intents, creating a benchmark we call \textit{ElementWorld}.
\textit{ElementWorld} is a randomized gridworld environment with a cylindrical topology (\cref{fig:element-world}, left).
An agent starts uniform randomly in the bottom row of the map, and must navigate to any of the terminal goal states at the top of the map.
Between the start and goal areas, we generate `lanes' corresponding to each of $E$ elements (e.g. water, dirt, grass), which are randomly perturbed horizontally, forcing the agent to avoid obstacles, leading to non-trivial demonstrations.
The agent is randomly assigned one of $E$ possible intents, and for the $k^\text{th}$ intent, we set the ground truth reward parameters for all $j \ne k$ elements to $-10$, the goal to $0$, and the $k^\text{th}$ element and the start zone to $-1$.
For actions, the agent chooses from the four cardinal directions, however this fails (and a random choice is made) with a wind probability $w$.
We use an $E + 2$ dimensional state feature function $\phi(s) = (\text{Start}, \text{Goal}, \text{Element}_1, \dots, \text{Element}_E)$, which is a one-hot representation of the type of cell the agent is in.

By adjusting the number of elements $E$ and the wind probability $w$ we can vary the difficulty of the MI-IRL behaviour clustering sub-problem: $w \rightarrow 0$ leads to almost disjoint intent clusters in feature space, while $w \rightarrow 1$ leads to almost completely overlapping intents in feature space.
Likewise, by adjusting the height $h$ of the map, we change the length of the demonstrations, adjusting the difficulty of MI-IRL reward learning sub-problem.

\begin{table}[t]
    \centering
    \caption{%
        LiMIIRL experimental results for \textit{EW} -- \textit{ElementWorld}, and \textit{DB} -- Driver Behaviour forecasting.
		Values are mean $\pm ~95\%$ confidence interval over 100 (\textit{EW}), 8 (\textit{DB}) repeat experiments with randomized seeds.
		Lower is better for all metrics.
    }
    \resizebox{\textwidth}{!}{%
        %
        %
        \begin{tabular}{c|rrrrrr}
        \toprule
         & Algorithm  & Iterations & Duration (s,m) & NLL & ANID & GEVD
        \\
        \midrule
        \multirow{4}{*}{\rotatebox[origin=c]{90}{\textit{EW}}} & LiMIIRL-MLE &  \textbf{3.41 $\pm$ 01.40} &           33.20 $\pm$ 13.49 &  \textbf{19.65 $\pm$ 00.64} &  \textbf{0.03 $\pm$ 00.01} &           3.70 $\pm$ 01.11 
        \\
        & LiMIIRL-Mean &           5.08 $\pm$ 01.13 &  \textbf{21.86 $\pm$ 06.65} &           19.76 $\pm$ 00.45 &           0.04 $\pm$ 00.01 &           4.49 $\pm$ 00.98 
        \\
        & Random &          17.22 $\pm$ 04.80 &          136.58 $\pm$ 39.69 &           20.05 $\pm$ 00.71 &           0.13 $\pm$ 00.07 &           6.75 $\pm$ 02.09 
        \\
        & Supervised &                        N/A &                         N/A &           19.65 $\pm$ 00.66 &           0.04 $\pm$ 00.01 &  \textbf{3.65 $\pm$ 01.27} 
        \\
        \midrule
        \multirow{3}{*}{\rotatebox[origin=c]{90}{\textit{DB}}} & LiMIIRL-MLE  &           \textbf{1.17 $\pm$ 00.60} &           \textbf{68.29 $\pm$ 34.01} &           305.55 $\pm$ 01.20 & N/A & N/A
        \\
        & LiMIIRL-Mean &           4.17 $\pm$ 01.10 &         243.28 $\pm$ 64.22 &           \textbf{303.99 $\pm$ 03.17} & N/A & N/A
        \\
        & Random &           3.33 $\pm$ 01.51 &           198.40 $\pm$ 79.69 &           305.37 $\pm$ 03.17 & N/A & N/A
        \\
        \bottomrule
        \end{tabular}
    }%
    \label{tab:ew-results}
    \label{tab:porto-results}
\end{table}

\subsection{ElementWorld Results}
\label{subsec:element-world-results}

\vspace{1em}
\paragraph{Experimental configuration.}
We performed an experiment to validate our LiMIIRL algorithm.
With $E=3, w=0.1, h=6, \gamma=0.99$ we used Q-value iteration to find an optimal stochastic policy for each behaviour intent, then sampled uniformly from the policy ensemble for a dataset of $N=100$ demonstrations.
Using a MaxEnt IRL reward model we compared the performance of the EM algorithm with random initialization \cite{Babes-Vroman2011} to our LiMIIRL method with $k$-means clustering, with either mean, or MLE reward initialisation.
As a baseline, we also included a supervised reward ensemble that is provided with the labels mapping demonstrations to elements, and the EM MI-IRL algorithm with random initialisation, as used in \cite{Babes-Vroman2011}.
For this initial experiment we assumed the true number of elements was known \apriori{} (\ie{} $K=E$).
For each experiment, we recorded the number of iterations and wall time to EM convergence, the Negative Log Likelihood (NLL) and clustering performance (ANID) on a held-out testing set of 100 demonstrations, as well as the reward ensemble quality (GEVD).
Every experiment was repeated $100\times$ with different randomly generated \textit{ElementWorld} instances to allow estimating the mean and $95\%$ confidence interval for each metric.

\paragraph{No significant difference between MLE and mean initialisation.}
The results are summarized in \cref{tab:ew-results}, with corresponding box-plots included in the supplementary data file.
The results confirm our hypothesis that LiMIIRL is able to identify good starting points for the MI-IRL EM algorithm, at least for this problem instance which satisfies our theoretical assumptions.
The results also show that MLE and mean reward initialisation are roughly equivalent in this case -- both empirically seem to identify good starting points for the MI-IRL problem, with no statistically significant difference in the performance of these methods.
In general however, we have observed the MLE initialization seems more robust when there is large deviation from our theoretical assumptions, an idea we investigate further in the driver behaviour experiment below.

\paragraph{LiMIIRL avoid local-minima EM solutions.}
With hard clustering, LiMIIRL learns a better reward ensemble than the EM algorithm with random initialisation ($90\%$ relative reduction of ANID, $45\%$ relative reduction in GEVD), despite the solutions having equivalent Negative Log Likelihood.
This reveals an interesting property of the MI-EM algorithm -- the optimization landscape of this problem has many local minima which appear equal or nearly equal in NLL, but vary greatly in terms of actual performance, highlighting the importance of a good initialization point for the ensemble.

\paragraph{LiMIIRL leads to faster EM convergence.}
Our key finding however is that LiMIIRL leads to the EM algorithm reaching convergence significantly faster.
On average, LiMIIRL reduces the number of iterations and running time by around 80\% and 75\% respectively with $k$-means initialization, and 
30\% and 73\% respectively with GMM initialization.
Comparison with the corresponding Iteration and Duration box-plots reveals that the convergence metric results are heavy tailed -- especially when random initialisation is used, further supporting the idea that a good starting point is critical for this problem.

These results provide hope that LiMIIRL will be especially helpful on larger problems where each EM iteration is itself a computationally expensive procedure, provided we have reason to believe the behaviour intents are at least approximately separated in feature space.

\subsection{Sensitivity Analysis}
\label{subsec:sensitivity-analysis}

We conducted extensive additional experiments to investigate the sensitivity of our algorithm under additional problem variations, especially cases where the theoretical assumptions are violated.
We evaluated our algorithm under scenarios with varying numbers of
demonstrations, degree of cluster separation in feature space, number of learned
clusters, number of ground truth elements, cluster imbalance, initial clustering
method (hard $k$-means vs. soft Gaussian Mixture Models), as well as with a
different behaviour models (ML-IRL, and $\Sigma$-GIRL \cite{Babes-Vroman2011,Ramponi2020}).
For each algorithm and each problem instance, we performed over 20 repeat experiments with randomized ElementWorld seeds to generate means and $95\%$ confidence intervals.
For brevity, we detail the nature of each parameter sweep and report the full results as figures in the supplementary file, however we briefly describe a few key findings here.

\paragraph{Hard-clustering vs. soft-clustering.}
We observed that in general, hard clustering with KMeans always tends to out-perform soft clustering with a GMM.
Surprisingly, this is true even as the problem dynamics are made more stochastic (increasing wind factor $w$, which causes the intent clusters to intermingle in feature space).
We hypothesise that the hard clustering leads to less mis-classification of training demonstrations, which may lead to higher-quality rewards to be learned.
Encouragingly, hard clustering always performs on-par with the supervised baseline, with the exception of the case of large numbers of ground truth elements $E$, where there is a performance difference.

\paragraph{Number of learned clusters.}
As our algorithm is parametric, we require that the number of clusters $K$ be specified \apriori{}.
In experiments where we varied the number of learned clusters, keeping $E = 3$ fixed, we observe that the mixture NLL generally increases proportional to $K$, however the ANID and GEVD metrics reach a minima / asymptote respectively at $K=3$ -- this confirms that our metrics are a good measure of ensemble performance, even for the case when the number of learned clusters is mis-specified.
We note that in practice, the value of $K$ can be estimated using heuristics such as the Bayesian Information Criteria, using a non-parametric clustering as a pre-processing step, or using cross-fold validation.

\paragraph{Cluster Imbalance.}
Finally, we also tested the case of non-balanced behaviour intents in the dataset.
Encouragingly, LiMIIRL appears able to learn accurate ensembles in this case, correctly allocating less probability mass to the rarer behaviour intents.

\paragraph{Alternate IRL models.}
LiMIIRL appears to work with ML-IRL and $\Sigma$-GIRL as the IRL model too, however due to the computational complexity of these methods we did not include them in our driver forecasting experiment, below.

\section{Driver Forecasting Experiment}
\label{sec:porto-problem}

We demonstrate the merit of our method by application to a large driver behaviour dataset, consisting of Taxi GPS trajectories collected between 2013--14 in the city of Porto, Portugal \cite{Moreira-Matias2013,Dua2017}.
Our data pre-processing follows that used in \cite{Ziebart2008,Snoswell2020}: a particle filter was used to discretize the GPS trajectories to the road network, giving a deterministic MDP with 292,604 states (road segments) and 594,933 actions (turns at intersections).
As state features, we used the type of road (highway, local street etc.), the number of lanes, as well as a geographical indicator for the region of the city where that road segment is located (Central, North, North east, etc.)
The feature function was obtained by concatenating the individual one-hot feature vectors and multiplying by the length of the road segment in km, resulting in a 24 dimensional state-based feature vector.
We trained all models on paths $\le 300$ road segments in length, and used held-out test sets for evaluation.

We qualitatively observed the presence of multiple behaviour intents in this dataset (e.g. \cref{fig:porto-multimodal}, left), however we have no ground truth knowledge of the number of clusters $K$.
Instead, as a pre-processing step we used a stick-breaking non-parametric GMM in feature space to estimate the appropriate number of clusters for this dataset as $K \approx 3$.
To accelerate the training process, we only perform a partial maximization at the M-step, by truncating each M-step of the EM algorithm after 50 objective function evaluations.
Provided the learned solution is a monotonic improvement in objective than that of the previous iteration the convergence conditions for the EM algorithm will still be satisfied \cite{Dempster1977,Ramachandran2007}.
Experiments were performed on a 12-core Ubuntu workstation with Intel Xeon E5-2760 v3 CPUs.
We summarize the number of iterations and wall time to EM convergence, as well as the final mixture NLL in \cref{tab:porto-results}.
Each value is a mean and $\pm 95\%$ confidence interval across 8 folds of cross validation (1000 / 8000 randomly sampled train / test demonstrations for each fold).

The results align with that of \textit{ElementWorld} -- LiMIIRL with MLE initialisation learns mixtures with similar NLL to random initialisation, however reaches convergence significantly faster - a saving of between 2.16 and 2.33 EM iterations, corresponding to a non-trivial reduction in training time of between 130 and 135 minutes.
We note that the Mean initialisation method does not reduce the convergence time for this problem, which we attribute to the fact that the cluster separation assumption is substantially violated in this dataset.

We also investigate the qualitative behaviour models learned by the algorithms.
The learned reward ensembles were largely similar regardless of initialisation method, with a few exceptions (see supplementary file).
Interestingly, we are also able to identify qualitatively meaningful trends within the learned ensembles -- e.g. it appears a majority (85\% -- 90\%) of taxi drivers in this dataset share in a preference for higher speed limits, and most prefer navigating on routes through the northern and western suburbs -- perhaps due to the presence of a large ring-road highway that spans this area.
The remaining reward functions differed in the geographic and road type preferences - \eg{} preferring eastern suburbs and/or road types other than highways.
This demonstrates that LiMIIRL can be applied on a large, real world dataset to efficiently learn qualitatively meaningful behaviour models in the form of IRL reward ensembles.

\section{Related Work}
\label{sec:related-work}

\vspace{1em}
\paragraph{Alternate IRL formulations.}
IRL is an active research area, with fundamental methods being extended in multiple directions.
Our work is complementary to many such extensions, including works considering multiple interacting agents \cite{Natarajan2010}, or agents with multiple sub-goals \cite{Michini2012,Michini2015,Ding2020}, skills \cite{Ranchod2015}, or options \cite{Henderson2018}, or using adversarial reward learning to scale to high dimensional problems \cite{Fu2018}, or using specialised divergence to match a single intent in a multi-intent dataset \cite{Ke2020}.
Unlike these works, we consider the problem of learning multiple rewards from a
dataset containing several unlabeled behaviour intents \cite{Babes-Vroman2011}.
We highlight that a related but distinct problem is that of Meta-IRL, where the goal is to learn a good meta-reward prior \cite{Gleave2018,Yu2019} or shared base reward \cite{Chen2020} that can be used to learn new rewards efficiently.

\paragraph{Multiple-Intent IRL methods.}
In the MI-IRL setting, methods can be partitioned into parametric (known number of clusters $K$) and non-parametric approaches ($K$ unknown).
While initialisation strategies such as ours may be helpful in Bayesian non-parametric models \cite{Choi2011} or hierarchical non-parametric clustering schemes \cite{Almingol2015}, we leave exploration of this avenue to future work, instead focusing on parametric methods modelled on the Expectation Maximization framework \cite{Dempster1977}.
Specifically, our work is based on the EM-based MI-IRL approach of \cite{Babes-Vroman2011}, which has also been extended to a full Bayesian treatment \cite{Dimitrakakis2011}, and to use model-free gradient based IRL methods \cite{Ramponi2020}.
While our theoretical analysis uses the MaxEnt IRL behaviour model
\cite{Ziebart2008,Snoswell2020}, our empirical experiments with
the ML-IRL model from \cite{Babes-Vroman2011} and
the $\Sigma$-GIRL model of \cite{Ramponi2020} suggest that our initialisation
method is effective with a range of behaviour models.

\paragraph{Initialisation methods for Expectation Maximization.}
Our initialization method is designed to alleviate the problem of learning local minima solutions --- an issue that the family of EM algorithms are well-known to suffer from (as well as slow convergence rates) \cite{Dempster1977,Vlassis2002,Balakrishnan2017}.
As such, the study of good initialisation methods for specific EM algorithms is
an ongoing area of research \cite{Vlassis2002,OHagan2012,OHagan2019}, as is work on finding alternate update rules to address the same issues \cite{Vlassis2002,Mena2020}.
Our work relates to these investigations, however we focus on EM algorithms that are specific to clustering problems, rather than general data-completion problems
--- specifically, we consider intent clustering as encountered in MI-IRL.

\paragraph{IRL evaluation measures}
In addition to developing an initialisation method for EM-based MI-IRL, we propose a novel metric for evaluating the quality of MI-IRL reward ensembles with respect to a known ground-truth ensemble.
Our method extends the popular Expected Value Difference (EVD) and Inverse Learning Error (ILE)\footnote{The ILE is equivalent to EVD with a specific choice of state distribution.} metrics for comparing individual reward pairs \cite{Levine2011,Choi2011}.
Specifically, our GEVD metric generalize these measures to the case of hard or soft reward ensembles with potentially different sizes.
Previous approaches for evaluating the quality of reward ensembles have suffered from two main limitations - they depend on comparing the likelihood or return of a specific test trajectory dataset \cite{Babes-Vroman2011,Almingol2015}, or if not, they require manually matching learned rewards with the corresponding ground truth reward \cite{Dimitrakakis2011,Ramponi2020}.
In contrast, our method does not require a held out testing dataset, and is able to automatically, elegantly and efficiently solve the ensemble matching problem through a min-cost-flow formulation.
Recently, there has been renewed interest in methods for comparing reward pairs (i.e. the single-intent IRL problem) without requiring policy evaluation \cite{Gleave2021}.
These reward pseudo-metrics are complementary to our work, in that a straight-forward application of the min-cost-flow formulation outlined in \cref{sec:mi-irl-performance-metrics} can also extend these evaluation standards to the Multiple-Intent case.
However, while such pseudo-metric can be theoretically and computationally attractive, enforcing symmetry for the learned reward and the ground-truth reward in the metric may be undesirable in the IRL setting, which is asymmetric in nature,
as in general the effect of using a reward $R$ on a domain with reward $R'$ is different from the effect of using a reward $R'$ on a domain with reward $R$.
For example, if $R$ is a constant reward function, and $R'$ is a sparse reward function,
then using $R'$ in a domain with reward function $R$ does not lead to sub-optimal behavior, 
while using $R$ in a domain with reward function $R'$ is usually sub-optimal.

\section{Conclusion}
\label{sec:conclusion}

We presented a simple warm-start strategy for learning MI-IRL models and
demonstrate its effectiveness with both theoretical and empirical analysis.
In addition, we developed a novel performance metric, GEVD, for evaluating a learned
reward ensemble against a ground-truth one.
These contributions further MI-IRL research by developing simple strategies that
can make MI-IRL methods more efficient and practical, with insight on when the
methods can work well.

GEVD provides a natural and direct way to measure the quality of a learned
reward ensemble, but is limited to use with benchmark problems where the ground
truth reward ensemble is known, and policy iteration and evaluation can be used
to compute the regret of policies.
Further research could develop metrics that enjoys the benefits of GEVD but
circumvents these limitations.

Our warm-start strategy requires only mild assumptions, and enjoys provable
provable theoretical guarantee when using the mean initialization and the MaxEnt
IRL behaviour model.
Interestingly, empirical results suggest that the idea is effective for the MLE
initialisation and some other IRL models.
Theoretical results for these cases are non-trivial but will be an interesting
area for future work.
Another interesting area for future work is exploring if initialisation methods such
as ours can also be beneficial in parametric latent variable adversarial models
that are gaining interest in the Imitation Learning literature where policy
learning, not reward learning is the goal \cite{Li2017,Gruver2020}.
Finally, nonlinear rewards, particularly those represented by deep neural
networks, have been explored in various IRL works (e.g., see \cite{Fu2018}).
We can also extend the idea in this paper to provide a warm-start solution when
a pre-trained feature map is available.

\begin{ack}
Aaron Snoswell is supported by through an Australian Government Research Training Program Scholarship.
\end{ack}

\bibliographystyle{plainnat}
\bibliography{bibliography}

\clearpage
\appendix


\section{Theoretical Analysis of the LiMIIRL algorithm -- additional results}

Here we state an additional lemma that is useful for our analysis, and include proofs that are excluded from our main text.

The following simple result is handy useful in the below proofs.

\addtocounter{lemma}{2}
\begin{lemma} \label{lem:basic}
    Let $A = \{a_{1}, \ldots, a_{n}\}$,
    $A' = \{a_{1}', \ldots, a_{n}'\}$
    $B = \{b_{1}, \ldots, b_{m}\}$ and
    $B' = \{b_{1}', \ldots, b_{m}'\}$ be
    finite sets of positive numbers.
    If $\frac{a_{i}}{b_{j}} \ge \frac{a_{i}'}{b_{j}'} \alpha$ for all 
    $1 \le i \le n$,
    $1 \le j \le m$, then for 
    $I \subseteq \{1, \ldots, n\}$,
    $J \subseteq \{1, \ldots, m\}$, we have
    \begin{align}
        \frac{\sum_{i \in I} a_{i}}
        {\sum_{i \in I} a_{i} + \sum_{j \in J} b_{j}}
        &\ge 
        \frac{\sum_{i \in I} a_{i}' \alpha}
        {\sum_{i \in I} a_{i}' \alpha + \sum_{j \in J} b_{j}'} \\
        \frac{\sum_{j \in J} b_{j}}
        {\sum_{i \in I} a_{i} + \sum_{j \in J} b_{j}}
        &\le 
        \frac{\sum_{j \in J} b_{j}'}
        {\sum_{i \in I} a_{i}' \alpha + \sum_{j \in J} b_{j}'}
    \end{align}
\end{lemma}
\begin{proof}
    From the assumptions, we have
    \begin{align*}
        \frac{\sum_{i \in I} a_{i}}{b_{j}}
        =
        \sum_{i \in I} \frac{a_{i}}{b_{j}}
        \ge
        \sum_{i \in I} \frac{a_{i}'}{b_{j}'} \alpha
        =
        \frac{\sum_{i \in I} a_{i}' \alpha}{b_{j}'}.
    \end{align*}
    Inverting the fractions on both sides, we have
    \begin{align*}
        \frac{b_{j}}{\sum_{i \in I} a_{i}}
        \le 
        \frac{b_{j}'}{\sum_{i \in I} a_{i} \alpha}.
    \end{align*}
    Summing over $j \in J$ and adding 1 to both sides, we have
    \begin{align*}
        \frac{\sum_{i \in I} a_{i} + \sum_{j \in J} b_{j}}{\sum_{i \in I} a_{i}}
        \le 
        \frac{\sum_{i \in I} a_{i}' \alpha + \sum_{j \in J} b_{j}'}{\sum_{i \in I} a_{i} \alpha}.
    \end{align*}
    Inverting the fractions on both sides, we obtain the desired inequality
    \begin{align*}
        \frac{\sum_{i \in I} a_{i}}{\sum_{i \in I} a_{i} + \sum_{j \in J} b_{j}}
        \ge 
        \frac{\sum_{i \in I} a_{i} \alpha}{\sum_{i \in I} a_{i}' \alpha + \sum_{j \in J} b_{j}'}.
    \end{align*}
    The second inequality in the lemma thus follows easily by subtracting both sides of the inequality by 1.
\end{proof}

\begin{proof}[Proof of Lemma 1.]
    For any $\tau \in \Tau_{k}^{+}$ and any $\tau' \in \Tau_{\sim k}$, we have
    \begin{align*}
        \Phi_{k}^{\top} (\phi(\tau) - \phi(\tau')) 
        &= \frac{1}{|C_{k}|} \sum_{\tau'' \in C_{k}} \phi(\tau'')(\phi(\tau) - \phi(\tau'))
        \\
        &= \frac{1}{|C_{k}|} \sum_{\tau'' \in \Tau_{k}^{+}} \phi(\tau'')(\phi(\tau) - \phi(\tau'))
        + \frac{1}{|C_{k}|} \sum_{\tau'' \notin \Tau_{k}^{+}} \phi(\tau'')(\phi(\tau) - \phi(\tau'))
        \\
        &\ge (1-\delta) \intermargin + \delta (-\intermargin)
        \\
        &= (1 - 2 \delta) \intermargin,
    \end{align*}
    where the inequality follows from Assumption 1(a)(i) and
    Assumption 2.
    
    For any $\tau \in \Tau_{k}^{+}$ and any $\tau' \in \Tau_{k}^{-}$, we have
    \begin{align*}
        \Phi_{k}^{\top} (\phi(\tau) - \phi(\tau')) 
        &= \frac{1}{|C_{k}|} \sum_{\tau'' \in C_{k}} \phi(\tau'')(\phi(\tau) - \phi(\tau'))
        \\
        &= \frac{1}{|C_{k}|} \sum_{\tau'' \in \Tau_{k}^{+}} \phi(\tau'')(\phi(\tau) - \phi(\tau'))
        + \frac{1}{|C_{k}|} \sum_{\tau'' \notin \Tau_{k}^{+}} \phi(\tau'')(\phi(\tau) - \phi(\tau'))
        \\
        &\ge (1-\delta) \intramargin + \delta (-\intramargin) \\
        &=  (1-2\delta) \intramargin.
    \end{align*}
\end{proof}

\begin{proof}[Proof of Lemma 2.]
    Let $w(\tau) = q(\tau) \exp(\phi(\tau)^{\top} \Phi_{k})$,
    $W_{k} = \sum_{\tau \in \Tau_{k}} w(\tau)$,
    $W_{\sim k} = \sum_{\tau \in \Tau_{\sim k}} w(\tau)$,
    $W_{k}^{+} = \sum_{\tau \in \Tau_{k}^{+}} w(\tau)$,
    $W_{k}^{-} = \sum_{\tau \in \Tau_{k}^{-}} w(\tau)$, and 
    $W = \sum_{\tau \in \Tau} w(\tau)$.
    We obtain a lower bound on 
    \begin{align*}
        p(\tau \mid \Phi_{k})
        =
        \frac{w(\tau)}
        {W}
        =
        \frac{W_{k}}{W}
        \frac{w(\tau)}{W_{k}}
    \end{align*}
    by lower bounding
    $\frac{W_{k}}{W}$
    and 
    $\frac{w(\tau)}{W_{k}}$
    respectively.
    
    To lower bound $\frac{W_{k}}{W}$, first note that $W_{k} \ge W_{k}^{+}$, thus 
    \begin{align}
        \frac{W_{k}}{W} 
        =
        \frac{W_{k}}{W_{k} + W_{\sim k}}
        \ge 
        \frac{W_{k}^{+}}{W_{k}^{+} + W_{\sim k}},
        \label{lem:lu:eq:1}
    \end{align}
    and it suffices to lower bound the last fraction.
    From Lemma 1, $\Phi_{k}$ $\tilde{\intermargin}$-separates $\Tau_{k}^{+}$ and
    $\Tau_{\sim k}$, thus for any $\tau \in \Tau_{k}^{+}$ and any $\tau' \in \Tau_{\sim k}$, we have
    \begin{align*}
        \frac{w(\tau)}{w(\tau')} 
        = 
        \frac{q(\tau) \exp((\phi(\tau) - \phi(\tau'))^{\top} \Phi_{k})}
        {q(\tau')}
        \ge 
        \frac{q(\tau)}
        {q(\tau')} 
        \exp(\tilde{\intermargin}),
    \end{align*}
    where the equality is obtained by dividing both numerator and
    denominator by $\exp(\phi(\tau)^{\top} \Phi_{k})$. 
    Using \cref{lem:basic}, we have
    \begin{align}
        \frac{W_{k}^{+}}
        {W_{k}^{+} + W_{\sim k}}
        &=
        \frac{\sum_{\tau \in \Tau_{k}^{+}} w(\tau)}
        {\sum_{\tau' \in \Tau_{k}^{+}} w(\tau')
            + \sum_{\tau' \notin \Tau_{k}} w(\tau')} \nonumber \\
        &\ge 
        \frac{\sum_{\tau \in \Tau_{k}^{+}} q(\tau) \exp(\tilde{\intermargin})}
        {\sum_{\tau' \in \Tau_{k}^{+}} q(\tau') \exp(\tilde{\intermargin})
            + \sum_{\tau' \in \Tau_{\sim k}} q(\tau')} \nonumber \\
        &=
        \frac{
            Q_{k}^{+} \exp(\tilde{\intermargin})
        }{
            Q_{k}^{+} \exp(\tilde{\intermargin}) + Q_{\sim k}
        } \nonumber \\
        &= 
        \frac{(1-\zeta) Q_{k} \exp(\tilde{\intermargin})}
        {(1-\zeta) Q_{k} \exp(\tilde{\intermargin}) + Q_{\sim k}}.
        \label{lem:lu:eq:2}
    \end{align}
    
    Now, to lower bound 	
    $\frac{w(\tau)}
    {W_{k}}$, note that since $\Phi_{k}$
    $\tilde{\intramargin}$-separates $\Tau_{k}^{+}$ and $\Tau_{k}^{-}$, we have
    for any $\tau \in \Tau_{k}^{+}$ and any $\tau' \in \Tau_{k}^{-}$
    $\frac{w(\tau)}{w(\tau')} 
    =
    \frac{q(\tau) \exp((\phi(\tau) - \phi(\tau'))^{\top} \Phi_{k})}
    {q(\tau')}
    \ge 
    \frac{q(\tau)}{q(\tau')} e^{\tilde{\intramargin}}$.
    Hence 
    \begin{align*}
        \frac{W_{k}^{+}}
        {W_{k}^{-}}
        =
        \frac{\sum_{\tau \in \Tau_{k}^{+}} w(\tau)} 
        {\sum_{\tau \in \Tau_{k}^{-}} w(\tau)} 
        \ge 
        \frac{\sum_{\tau \in \Tau_{k}^{+}} q(\tau)}
        {\sum_{\tau \in \Tau_{k}^{-}} q(\tau)} e^{\tilde{\intramargin}}
        \ge 
        \frac{(1-\zeta) Q_{k}}
        {\zeta Q_{k}} e^{\tilde{\intramargin}}
        =
        \frac{(1-\zeta) e^{\tilde{\intramargin}}}
        {\zeta}.
    \end{align*}
    Thus we have 
    $\frac{\zeta + (1-\zeta) e^{\tilde{\intramargin}}}
    {(1-\zeta) e^{\tilde{\intramargin}}}
    W_{k}^{+} 
    \ge
    W_{k}$.
    As a consequence, 
    \begin{align}
        \frac{w(\tau)}{W_{k}}
        \ge 
        \frac
        {(1-\zeta) e^{\tilde{\intramargin}}}
        {\zeta + (1-\zeta) e^{\tilde{\intramargin}}}
        \frac{w(\tau)}{W_{k}^{+}}.
        \label{lem:lu:eq:3}
    \end{align}
    On the other hand, we have
    \begin{align}
        \frac{
            w(\tau)
        }{
            W_{k}^{+}
        } &= \frac{
            q(\tau)
        }{
            \sum_{\tau' \in \Tau_{k}^{+}} q(\tau') \exp((\phi(\tau') - \phi(\tau))^{\top} \Phi_{k})
        }
        \nonumber
        \\
        &\ge \frac{
            q(\tau)
        }{
            \sum_{\tau' \in \Tau_{k}^{+}} q(\tau') \exp( 2r_{k} \norm{\Phi_{k}} )
        }
        \nonumber
        \\
        &=
        \frac{
            q(\tau) \exp(-2r_{k} \norm{\Phi_{k}})
        }{
            Q_{k}^{+}
        } \nonumber\\
        &\ge 
        \frac{
            q(\tau) \exp(-2r_{k} \norm{\Phi_{k}})
        }{
            Q_{k}
        } 
        \label{lem:lu:eq:4}
    \end{align}
    where the inequality follows from $(\phi(\tau') - \phi(\tau))^{\top} \Phi_{k}) \le \norm{\phi(\tau') - \phi(\tau)} \norm{\Phi_{k})} \le 2 r_{k} \norm{\Phi_{k}}$.
    Combining \cref{lem:lu:eq:1}-\cref{lem:lu:eq:4}, for any $\tau \in \Tau_{k}$, we have
    \begin{align*}
        p(\tau \mid \Phi_{k}) 
        &= 
        \frac{
            w(\tau)
        }{
            W
        }
        \ge 
        \frac{W_{k}}{W}
        \frac{w(\tau)}{W_{k}}
        \\
        &\ge 
        \frac{(1-\zeta) Q_{k} \exp(\tilde{\intermargin})}
        {(1-\zeta) Q_{k} \exp(\tilde{\intermargin}) + Q_{\sim k}}
        \frac
        {(1-\zeta) e^{\tilde{\intramargin}}}
        {\zeta + (1-\zeta) e^{\tilde{\intramargin}}}
        \frac{
            q(\tau) \exp(-2r_{k} \norm{\Phi_{k}})
        }{
            Q_{k}
        }  \\
        &=
        \frac{(1-\zeta) q(\tau) \exp(\tilde{\intermargin}-2r_{k} \norm{\Phi_{k}})}
        {(1-\zeta) Q_{k} \exp(\tilde{\intermargin}) + Q_{\sim k}}
        \frac
        {(1-\zeta) e^{\tilde{\intramargin}}}
        {\zeta + (1-\zeta) e^{\tilde{\intramargin}}}
    \end{align*}
    Now, for any $\tau \notin \Tau_{k}$, we have
    \begin{align*}
        p(\tau \mid \Phi_{k}) 
        &=
        \frac{w(\tau)}
        {\sum_{\tau' \in \Tau} w(\tau')}
        \\
        &\le 
        \frac{w(\tau)} 
        {\sum_{\tau' \in \Tau_{k}^{+}} w(\tau') + w(\tau)} \\
        &\le \frac{
            q(\tau)
        }{
            (1-\zeta) Q_{k} \exp(\tilde{\intermargin}) + q(\tau)
        },
    \end{align*}
    where the last inequality follows from \cref{lem:basic}.
\end{proof}

\begin{proof}[Proof of Theorem 1]
    Consider an arbitrary $\tau_{i}$.
    Assume that $\tau_{i} \in C_{k}$, and consider $k' \neq k$, then by Lemma 2,
    \begin{align*}
        p(\tau_{i} \mid \Phi_{k}) 
        &\ge 
        \frac{(1-\zeta) q(\tau_{i}) e^{\tilde{\intermargin}-2r_{k} \norm{\Phi_{k}}}}
        {(1-\zeta) Q_{k} e^{\tilde{\intermargin}} + Q_{\sim k}}
        \frac
        {(1-\zeta) e^{\tilde{\intramargin}}}
        {\zeta + (1-\zeta) e^{\tilde{\intramargin}}},
        \\
        p(\tau_{i} \mid \Phi_{k'}) 
        &\le 
        \frac{q(\tau_{i})}{(1 - \zeta) Q_{k'} e^{\tilde{\intermargin}} + q(\tau_{i})}.
    \end{align*}
    This gives us the following bound
    \begin{align*}
        \frac{\rho_{k}^{(0)} p(\tau_{i} \mid \Phi_{k})}
        {\rho_{k'}^{(0)} p(\tau_{i} \mid \Phi_{k'})}
        \ge
        \frac
        {(1-\zeta)^{2} e^{\tilde{\intramargin}}}
        {\zeta + (1-\zeta) e^{\tilde{\intramargin}}}
        \frac{|C_{k}| ((1-\zeta) Q_{k'} e^{\tilde{\intermargin}} + q(\tau_{i}))}
        {|C_{k'}| ((1-\zeta) Q_{k} e^{\tilde{\intermargin}} +Q_{\sim k})}
        e^{\tilde{\intermargin} - 2r_{k} \norm{\Phi_{k}}}
        \ge \beta.
    \end{align*}
    Hence we have 
    \begin{align*}
        u_{ik}^{(1)} &= \frac{
            \rho_{k}^{(0)} p(\tau_{i} \mid \theta_{k}^{(0)})
        }{
            \sum_{k'} \rho_{k'}^{(0)} p(\tau_{i} \mid \theta_{k'}^{(0)})
        },
        \\
        &= \frac{
            \rho_{k}^{(0)} p(\tau_{i} \mid \theta_{k}^{(0)})
        }{
            \rho_{k}^{(0)} p(\tau_{i} \mid \theta_{k}^{(0)}) + \sum_{k' \neq k} \rho_{k'}^{(0)} p(\tau_{i} \mid \theta_{k'}^{(0)})
        }, \\
        &\ge \frac{
            \beta
        }{
            \beta + K-1
        },
    \end{align*}
    where the inequality is obtained with the help of Lemma 3 (in supplementary file).
    Since $u_{ik}^{(0)} = 1$, we have
    \begin{align}
        \sum_{k'} |u_{ik'}^{(1)} - u_{ik'}^{(0)}| &= 2 |u_{ik}^{(1)} - 1|
        \\
        &\le \frac{
            2(K-1)
        }{
            \beta + K - 1
        }.
    \end{align}
    Since at least $1-\delta$ of the trajectories are typical, and for the
    remaining trajectories, $\sum_{k'} |u_{ik'}^{(1)} - u_{ik'}^{(0)}| \le 2$, we
    have
    \begin{align*}
        \frac{1}{|\Data|} \sum_{i} \sum_{k} |u_{ik}^{(1)} - u_{ik}^{(0)}| 
        &\le 
        2\delta + 
        (1-\delta)
        \frac{2 (K-1)}{\beta + K - 1}.
    \end{align*}
\end{proof}

\clearpage
\section{Societal Impact Discussion}

Our work is mainly algorithmic and theoretical and therefore somewhat indirectly linked to potential negative societal impacts - however we briefly discuss such potential impacts here.

\paragraph{Methodology.}
On the methodology side, our LiMIIRL algorithm could help enable the application of MI-IRL methods to larger or more substantive real-world problems, and our GEVD metric also provides insights on when these Multi-Intent techniques may work well.

Such Multi-Intent IRL methods allow for teasing-out individual user preferences in a cohort of un-labelled data - as such, they might be considered to pose a slightly higher risk in terms of surveillance or revelation of user preferences, compared to previous work on the (single intent) IRL problem.
However for the most part, our present work mostly magnifies any existing potential negative or positive impacts of IRL methods generally.
The most common IRL algorithms function by observing state and action sequences, which typically requires consent of the subject (c.f. state-only observations which can often be collected by surveillance in the absence of consent or even in adversarial contexts) - we consider this to be a strong mitigating factor regarding the potential negative impacts of IRL algorithms as a class of machine learning models.
Furthermore, in our work, we are also concerned with learning simple (linear) rewards for the purpose of understanding behaviour, rather than complicated (e.g. deep neural network) rewards for replicating behaviour through apprenticeship learning.
As such, our models do not directly interact with the real world in the sense that e.g. an apprenticeship learning agent might, but rather are useful for interrogating the world and building a better understanding of certain behaviours, which potentially limits or mitigates negative impact from this work.

\paragraph{Application.}
The example application that we consider (driver behaviour forecasting) is incidental to our methodological contributions and was selected in part due to it's low potential for negative societal impacts, and large potential for positive societal impact.

\clearpage
\section{Evaluation Measures for MI-IRL Algorithms}

We provide some additional comments and insight on our MI-IRL evaluation measures.
Implementations of the below measures are included in our open-source project repository.

\paragraph{Clustering Performance -- ANID metric.}

Given a set $\mathds{S}$ of $N$ items, a hard (or patitional) clustering divides the set into non-overlapping sub-sets, $\mathds{U} \triangleq \{\mathds{U}_1, \dots, \mathds{U}_{K_1}\}$ with $\bigcup_{i=1}^{K_1} \mathds{U}_i = N$ and $\mathds{U}_i \bigcap \mathds{U}_j = \emptyset ~\forall j \ne i$.
Given a second hard clustering $\mathds{V} \triangleq \{\mathds{V}_1, \dots, \mathds{V}_{K_2}\}$, the \textit{contingency table} $\bm{N}$ with elements $n_{ij} \triangleq \abs{\mathds{U}_i \bigcap \mathds{V}_j}$ and marginals $a_i \triangleq \sum_{j=1}^{K_2} n_{ij}$ and $b_j \triangleq \sum_{i=1}^{K_1} n_{ij}$ allows calculating various information-theoretic measures about the two clusterings,
\begin{align*}
    H(\mathds{U}) \triangleq -\sum_{i=1}^{K_1} \frac{a_i}{N} \log \frac{a_i}{N}, \quad&\quad
    H(\mathds{V}) \triangleq -\sum_{j=1}^{K_2} \frac{b_i}{N} \log \frac{b_i}{N},
    \\
    H(\mathds{U} \mid \mathds{V}) \triangleq -\sum_{i=1}^{K_1} \sum_{j=1}^{K_2} \frac{n_{ij}}{N} \log \frac{n_{ij} / N}{b_j / N}, \quad&\quad
    H(\mathds{V} \mid \mathds{U}) \triangleq -\sum_{i=1}^{K_1} \sum_{j=1}^{K_2} \frac{n_{ij}}{N} \log \frac{n_{ij} / N}{a_i / N},
    \\
    H(\mathds{U}, \mathds{V}) =
    H(\mathds{V}, \mathds{U}) &\triangleq -\sum_{i=1}^{K_1} \sum_{j=1}^{K_2} \frac{n_{ij}}{N} \log \frac{n_{ij}}{{N}},
    \\
    I(\mathds{U}, \mathds{V}) = 
    I(\mathds{V}, \mathds{U}) &\triangleq \sum_{i=1}^{K_1} \sum_{j=1}^{K_2} \frac{n_{ij}}{N} \log \frac{n_{ij} / N}{a_i b_j / N^2},
\end{align*}
where $H$ is a (Shannon) entropy, and $I$ is known as the \textit{mutual information} - a measure of how much information the two clustering share.

Various measures based on these terms have been proposed as ways to evaluate the similarity of hard clusterings, however recent work has highlighted the strengths (as compared to other metrics that have historically been popular, such as the clustering purity \citep{Zhao2001}, Rand Index (or adjusted variant) \citep{Rand1971}, and other information theoretic measures) of the max-Normalized Information Distance (NID) due to it's beneficial theoretical and empirical properties \citep{Vinh2010}.
The NID is given by,
\begin{align}
    \text{NID} &\triangleq 1 - \frac{
        I(\mathds{U},\mathds{V})
    }{
        \max \{ H(\mathds{U}),H(\mathds{V}) \}
    }.
\end{align}

In the context of multi-modal IRL, an algorithm may in-fact produce a so-called soft (probabilistic) assignment of demonstrations to behaviour modes.
A soft clustering of $N$ items into $K_1$ clusters can be represented by a \textit{responsibility matrix} $\bm{U} \in \mathds{R}^{N\times K_1}$ with elements $u_{ij} \ge 0$ and $\sum_{j=1}^{K_1} u_{ij} = 1 ~\forall i$.
Given a second soft clustering $\bm{V} \in \mathds{R}^{N \times K_2}$, work by \citep{Anderson2010} proposes an analogous contingency table, $\bm{N}^* \triangleq \bm{U}^\top \bm{V}$, which allows applying any metric based on contingency tables to the case of soft clusters, and subsequent work by \citet{Lei2014,Lei2017} shows that the NID is again a good choice of metric for evaluating soft clusterings.
In our experiments, we use the \textit{adjusted} Normalized Information Distance (ANID), which extends the NID to control for random chance - without this adjustment, the apparent agreement between clusters will be positively correlated with the number of clusters $K_1$ or $K_2$ \citep{Romano2016}.
The ANID is given by,
\begin{align}
    \text{ANID} &\triangleq 1 - \frac{
        I(\bm{U},\bm{V}) - \E[I(\bm{U}',\bm{V}')]
    }{
        \max \{ H(\bm{U}),H(\bm{V}) \} - \E[I(\bm{U}',\bm{V}')]
    },
\end{align}
where $\bm{U}'$ and $\bm{V}'$ are from an appropriate random clustering model -- for hard clusters hypergeometric models are commonly used, an in our soft clustering case we choose $\text{dim}(\bm{U}') = \text{dim}(\bm{U})$ and sample rows from the uniform Dirichlet distribution $u_i \sim \text{GEM}(1/K_1)$, and similarly for $\bm{V}'$.
The ANID has a number of useful properties - it is stochastically normalized to the range $[0, 1]$, with values approaching zero indicating perfect agreement between the two clusters.
Furthermore, in the limit $N \rightarrow \infty$, the ANID obeys the triangle equality and converges to a proper metric \citep{Romano2016}.

\paragraph{Reward Ensemble Performance -- GEVD metric.}

To evaluate single-intent IRL algorithms, \citet{Choi2011} proposed the \textit{Inverse Learning Error} (ILE), which was further refined by \citet{Bogert2016},
\begin{align}
    \text{ILE}(R_\text{GT}, R_\text{L}) &\triangleq \norm{
        \bm{v}(\pi^*_{R_\text{GT}}) - \bm{v}(\pi^*_{R_\text{L}})
    }_1,
    & \in [0, \infty)
\end{align}
where $\bm{v}(\pi)$ indicates the vector of state values \wrt{} the \emph{ground truth} reward $R_\text{GT}$ for any arbitrary policy $\pi$, and $\pi^*_{R_\text{GT}}$ and $\pi^*_{R_\text{L}}$ denote the optimal policy \wrt{} the ground truth and learned reward functions respectively.
The ILE is an appropriate measure of regret for the IRL context, however has the limitation that it grows with the number of states in the MDP, making it difficult to compare across different tasks, and difficult to apply to MDPs with continuous state spaces.

Instead, \citet{Levine2011} proposed a generalisation\footnote{EVD generalized ILE in the sense that ILE can be viewed as EVD computed under a uniform starting state distribution.} of this measure, known as the \textit{Expected Value Difference} (EVD),
\begin{align}
    \text{EVD}(R_\text{GT}, R_\text{L}) &\triangleq
    \E_{s \sim p_0} [
    v(\pi^*_{R_\text{GT}})
    ] - 
    \E_{s \sim p_0} [
    v(\pi^*_{R_\text{L}})
    ]
    & \in [0, \infty)
\end{align}
where $v(\pi)$ indicates the state-value function for the arbitrary policy $\pi$ \wrt{} the \emph{ground truth} reward $R_\text{GT}$, and $\pi^*_{R_\text{GT}}$ and $\pi^*_{R_\text{L}}$ denote the optimal policy \wrt{} the ground truth and learned reward functions respectively.
While the EVD has the same range as the ILE\footnote{\textit{N.b.} it is possible to upper-bound these single-intent IRL measures more tightly, using the same approach we apply to our GEVD metric (see the main text).}, the upped bound does not necessarily depend on the full set of states in the MDP, and we can more easily compute the EVD for continuous state-space MDPs.

Unlike a single-intent IRL method, MI-IRL algorithms output a set of learned reward function-weight pairs $\mathds{R}_\text{L} = \{ (w'_i, R_{\text{L}_i}) \}_{i=1}^{K_1}$, which we desire to compare with the set of ground truth reward functions and their corresponding weights $\mathds{R}_\text{GT} = \{ (w_j, R_{\text{GT}_j}) \}_{j=1}^{K_2}$.
To achieve this comparison, we find the most optimistic soft matching of learned and ground truth rewards.
Given $e_{ij}$, the matrix of EVD values for the $i$-th learned model \wrt{} the $j$-th ground truth model, we define the Generalized Expected Value Difference (GEVD) as,
\begin{align}
    \text{GEVD}(\mathds{R}_\text{GT}, \mathds{R}_\text{L}) &\triangleq
    \argmin_{\{w_{ij}\}: \sum_{j} w_{ij} =
        w_{i}, \sum_{i} w_{ij} = w'_{j}} \sum_{i, j} w_{ij} ~ e_{ij},
    & \in [0, \frac{\max_k \Delta_k}{1 - \gamma}]
\end{align}

\noindent where $\Delta_k = \max_{s, a, s'} R_{{GT}_k}(s, a, s') - \min_{s,a,s'} R_{{GT}_k}(s, a, s')$.
Computing the GEVD is equivalent to solving a minimum cost flow problem over a dense directed graph from learned to ground truth reward modes, where
\begin{enumerate}[(a)]
    \item there is a single unity demand source (S) and sink (T), and
    \item the inner and outer edges have zero cost and capacity equal to the weight of that mixture component, and
    \item the inside edges have unity capacity and costs equal to the pairwise EVD values,
\end{enumerate}

\noindent a construction we illustrate in \cref{fig:min-cost-flow}.

\begin{figure}[h]
    \centering
    \includegraphics[width=0.7\textwidth]{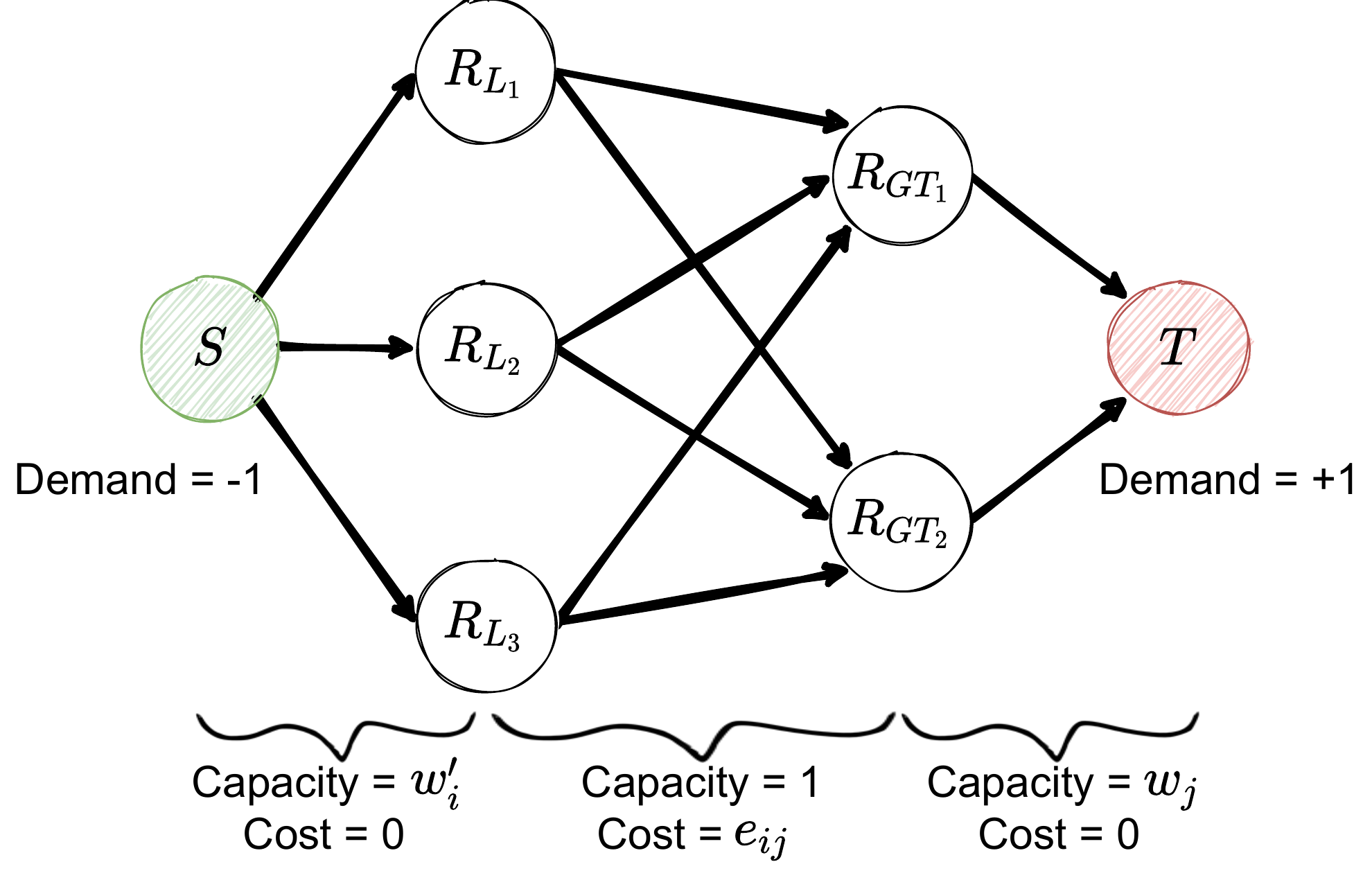}
    \caption{%
        The GEVD Score finds the most optimistic pairing between a learned mixture model $\{ \mathds{R}_{\text{L}_i} \}_{i=1}^{K_1}$ and the ground truth mixture model $\{ \mathds{R}_{\text{GT}_j} \}_{j=1}^{K_2}$.
    }
    \label{fig:min-cost-flow}
\end{figure}

\clearpage
\section[LiMIIRL algorithm - full results]{LiMIIRL algorithm - extended results}
\label{subsec:lim-results}

We provide additional results for LiMIIRL on the \textit{ElementWorld} problem when using 
different single-intent IRL models, namely, MaxEnt IRL, ML-IRL, and $\Sigma$-GIRL.
In each case, we run a MI-IRL algorithm multiple times and represent the runs using a boxplot. 
We use $100\times$ repeat experiments with randomized seeds for the MaxEnt IRL plots, and $20\times$ repeat experiments for the ML-IRL and $40\times$ repeat experiments for $\Sigma$-GIRL experiments (both of which are much more computationally demanding to run).
The boxes are centered at the median and extend to the first quartile, with whiskers extending to $1.5 \times$ the Inter Quartile Range.
Outlying values beyond $\pm 1.5$ IQR are shown with small diamonds.
For all metrics, lower values are better.
The results are shown in \cref{fig:ew-maxent,fig:ew-maxlik,fig:ew-sgirl}.

\paragraph{A note on $\Sigma$-GIRL Experiments}

The `Multiple-Intent' Expectation Maximization $\Sigma$-GIRL algorithm \cite{Ramponi2020} solves a similar but fundamentally different problem to that solved by Multiple-Intent Maximum Likelihood IRL \cite{Babes-Vroman2011} and Multiple-Intent Maximum Entropy IRL (the primary algorithm used in our paper).

Namely, $\Sigma$-GIRL addresses a partially labelled variant of the Multiple-Intent problem, where we have $N$ \textit{trajectories} generated by $M \le N$ \textit{agents}, who optimize for $K \le M$ particular behaviour \textit{intents}.
In $\Sigma$-GIRL, the mapping from trajectories to agents is known as part of the problem definition.
This information is used to form policy gradient Jacobian mean and covariance estimates (one for each agent), which are then provided to the EM process to cluster the agents into intents.

In contrast, our present work and the original MI ML-IRL algorithm by \cite{Babes-Vroman2011} consider a fully-unsupervised version of the MI-IRL problem, where there is no such notion of `agents'.
Although the fully-unsupervised setting can be seen as a special case of the partially supervised case with $M = N$ agents (i.e. one trajectory per agent), 
the $\Sigma$-GIRL algorithm cannot be directly applied in this context due to the need for more than one trajectory to estimate the policy gradient Jacobian covariance matrix.

Nevertheless, we can utilise the LiMIIRL initialisation strategy for the un-labelled part of the $\Sigma$-GIRL algorithm  (clustering agents into intents), and we include experimental results for this case below, run using code provided by the original authors.
Specifically, for each of the $M$ agents, we compute the empirical feature expectation using their (known) individual sets of demonstrations, then run a hard ($k$-means) or soft (GMM) clustering procedure to compute the initial responsibility matrix $\{u_{ik}^{(0)}\}$ allocating agents to intents.
We then use this responsibility matrix in lieu of the first E-step of the MI $\Sigma$-GIRL EM routine, which proceeds to iteratively allocate agent Jacobian matrices to behavioural intents.
This is functionally equivalent to the LiMIIRL MLE initialisation strategy, but using the $\Sigma$-GIRL single-intent IRL algorithm as the underlying behaviour model.
It is less clear how to apply the LiMIIRL Mean initialisation method with $\Sigma$-GIRL, however this may be possible by exploiting feature transformations to convert the MDP feature domain to the simplex reward parameter domain searched by the $\Sigma$-GIRL algorithm -- an idea we leave for future work.

Due to the semi-supervised nature of the $\Sigma$-GIRL algorithm, and it's prohibitive memory complexity, our $\Sigma$-GIRL experiments used the simple continuous \textit{PuddleWorld} domain from the original $\Sigma$-GIRL paper (Figure 2 from \cite{Ramponi2020}).

\clearpage
\thispagestyle{empty}

\begin{figure}[p]
    \begin{subfigure}[t]{0.48\linewidth}
        \centering
        \makebox[\textwidth][c]{\includegraphics[width=\linewidth]{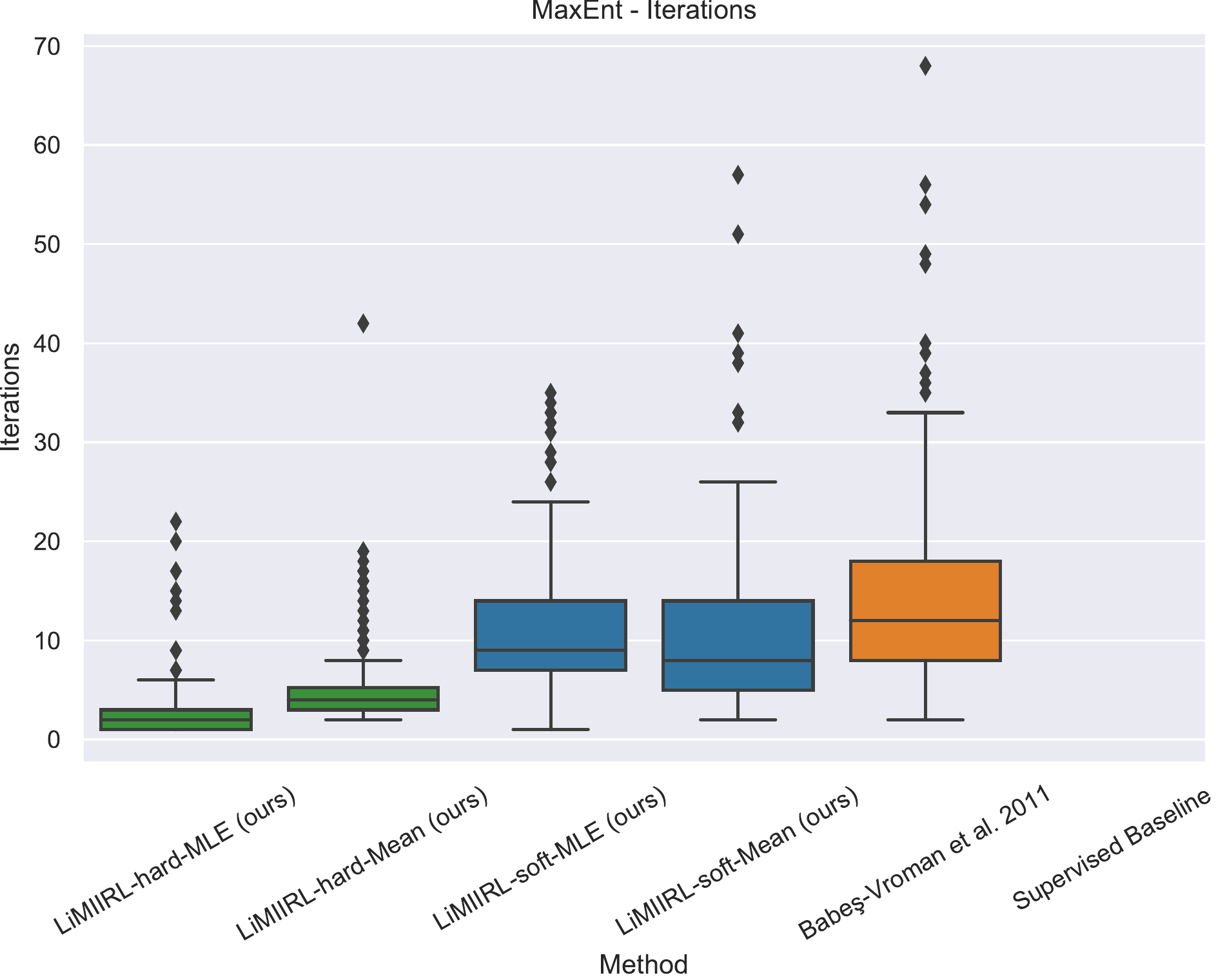}}
        \caption[]{%
            \textit{number of iterations}
        }
        \label{fig:ew-maxent-iterations}
        \vspace{20pt}
    \end{subfigure}
    \begin{subfigure}[t]{0.48\linewidth}
        \centering
        \makebox[\textwidth][c]{\includegraphics[width=\linewidth]{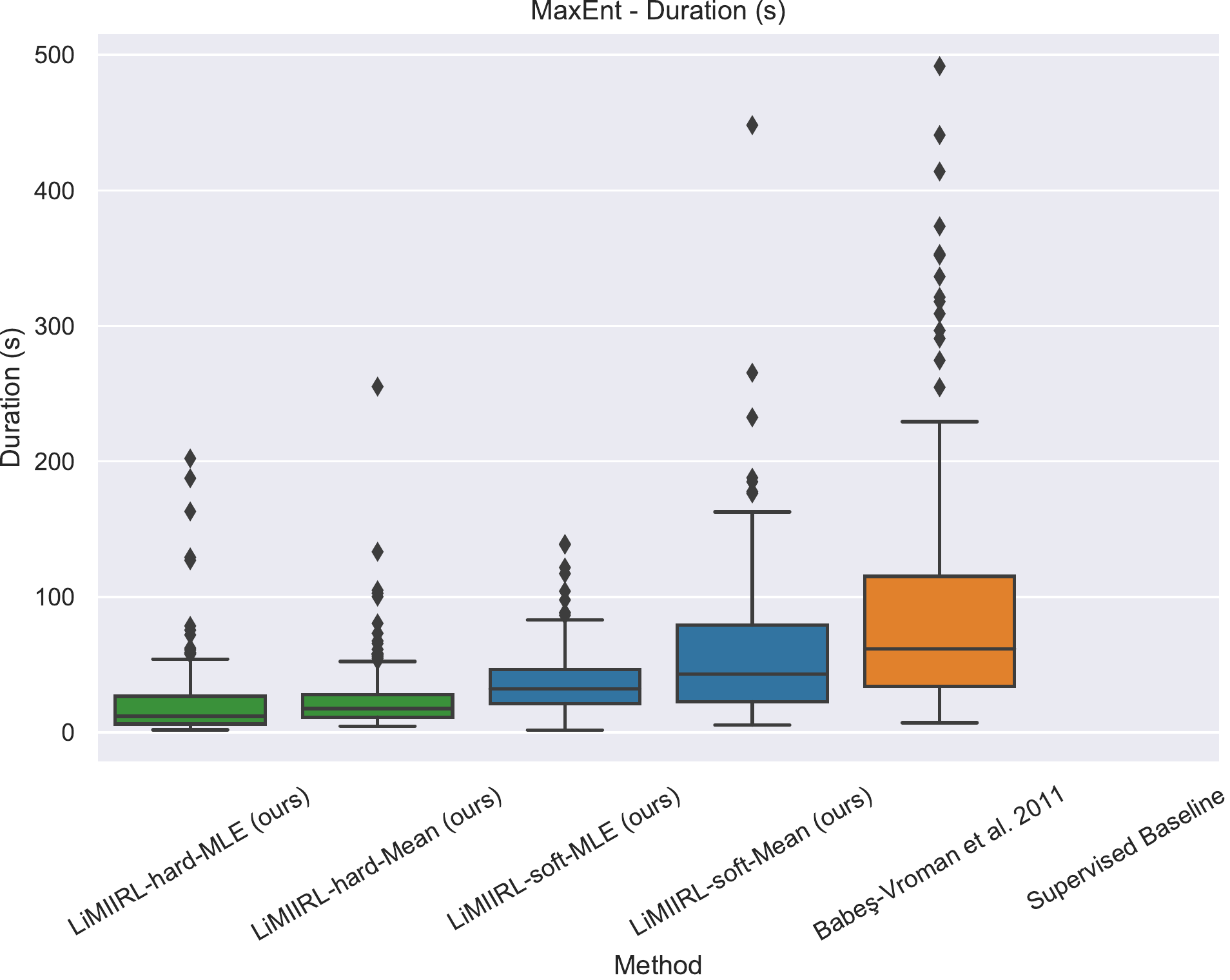}}
        \caption[]{%
            \textit{wall time in seconds}
        }
        \label{fig:ew-maxent-duration}
    \end{subfigure}
    
    \begin{subfigure}[t]{0.48\linewidth}
        \centering
        \makebox[\textwidth][c]{\includegraphics[width=\linewidth]{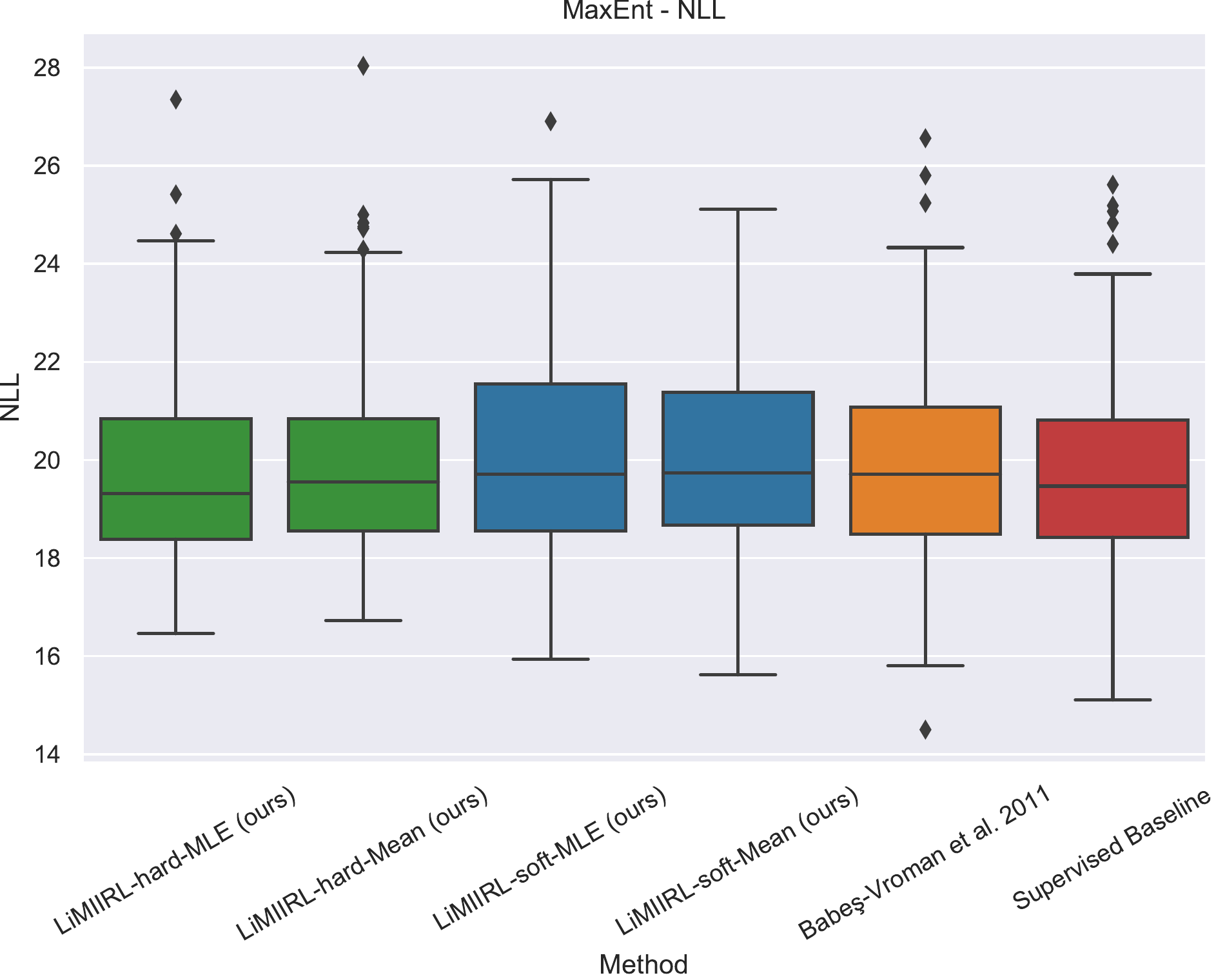}}
        \caption[]{%
            \textit{final mixture NLL}
        }
        \label{fig:ew-maxent-nll}
        \vspace{20pt}
    \end{subfigure}
    \begin{subfigure}[t]{0.48\linewidth}
        \centering
        \makebox[\textwidth][c]{\includegraphics[width=\linewidth]{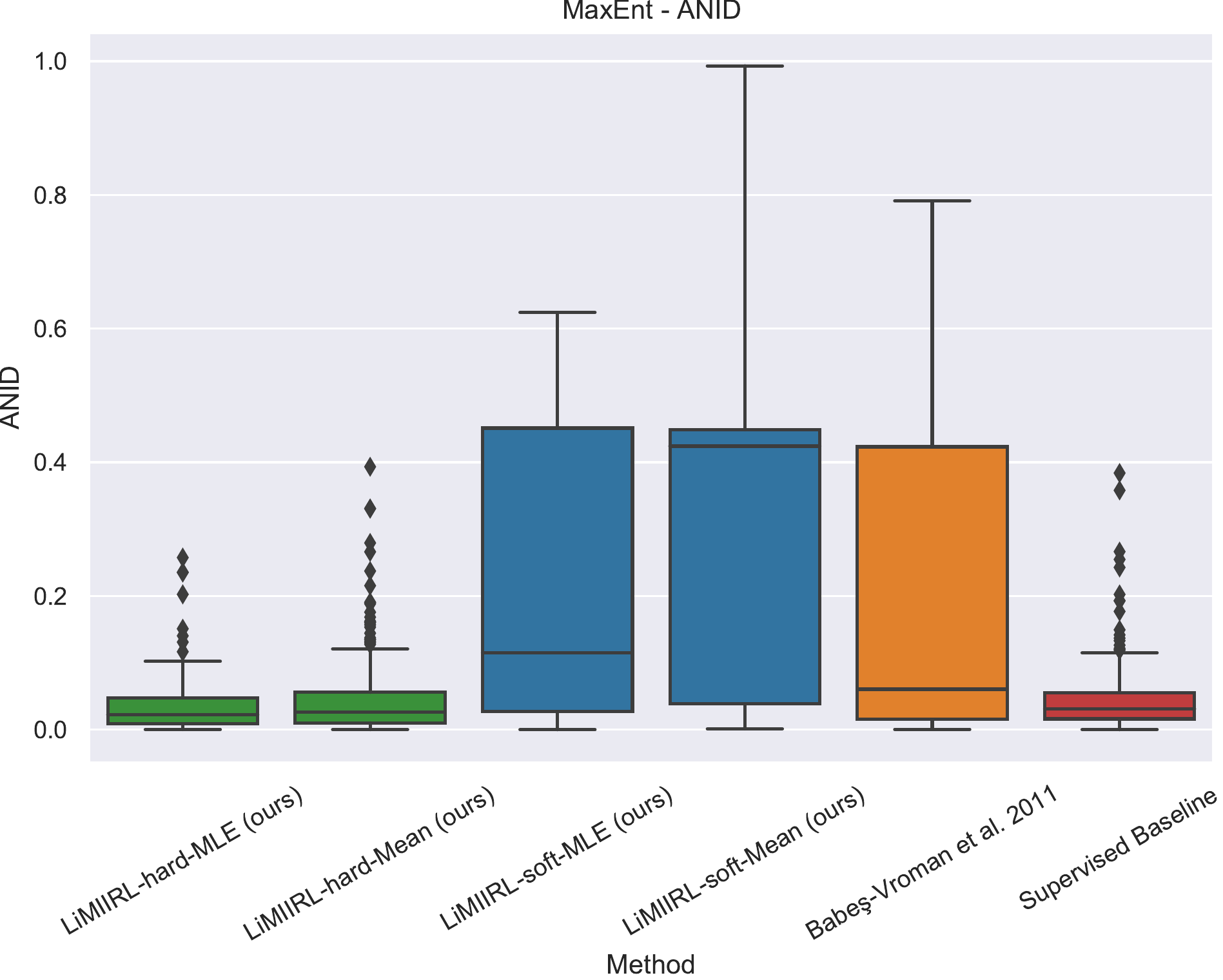}}
        \caption[]{%
            \textit{clustering performance (ANID)}
        }
        \label{fig:ew-maxent-anid}
    \end{subfigure}
    
    \begin{subfigure}[t]{0.48\linewidth}
        \centering
        \makebox[\textwidth][c]{\includegraphics[width=0.9\linewidth]{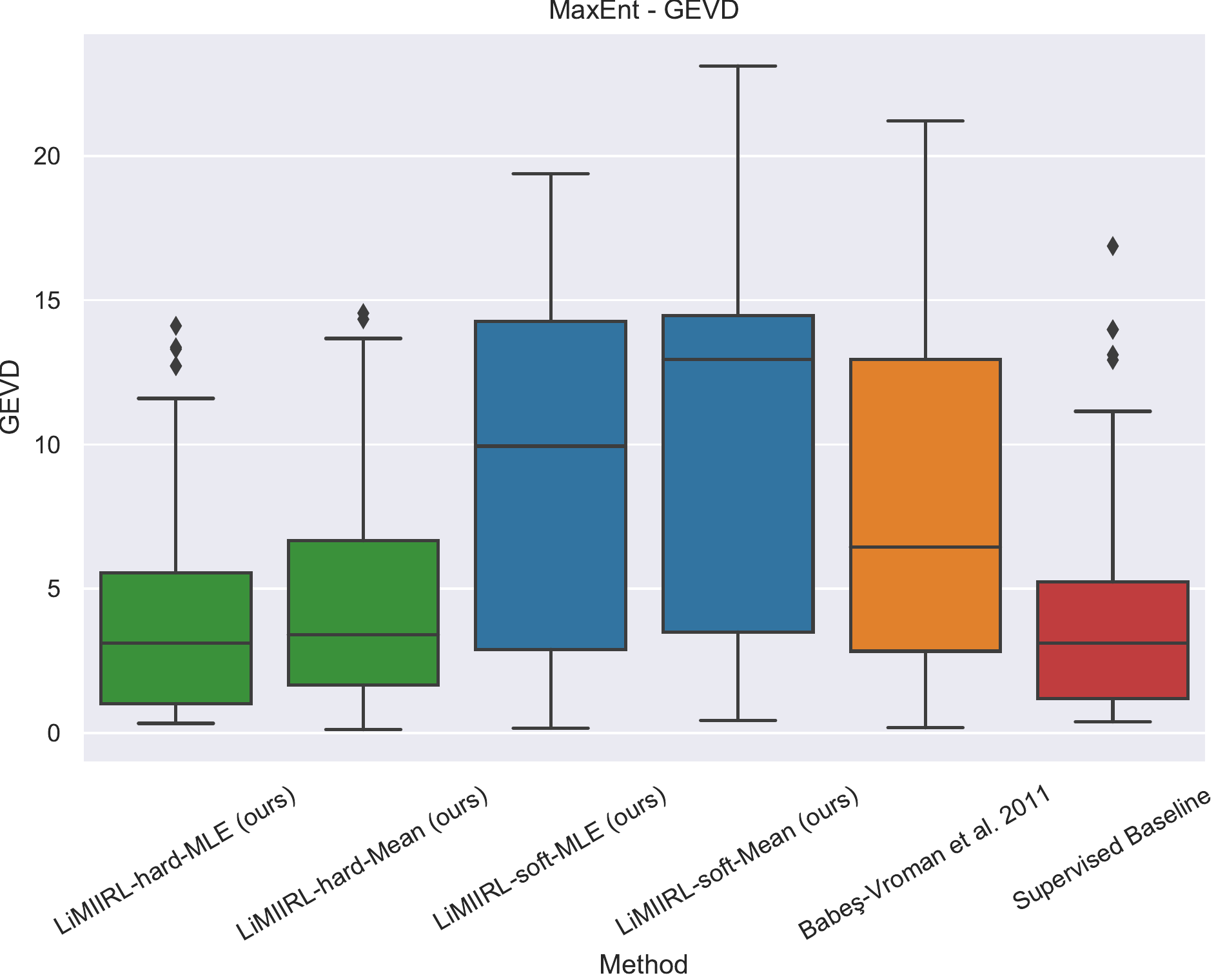}}
        \caption[]{%
            \textit{reward ensemble performance (GEVD)}
        }
        \label{fig:ew-maxent-gevd}
    \end{subfigure}
    \caption{
        \textit{ElementWorld} results with multi-intent MaxEnt IRL.
        The label `Babe\c{s}-Vroman et al. 2011' indicates EM with random initialisation.
        Although all methods result in a similar final NLL distribution, we highlight that there are significant differences in the other performance metrics, with the hard-clustering strategy consistently performing best.
    }
    \label{fig:ew-maxent}
\end{figure}

\clearpage
\thispagestyle{empty}

\begin{figure}[p]
    \begin{subfigure}[t]{0.48\linewidth}
        \centering
        \makebox[\textwidth][c]{\includegraphics[width=\linewidth]{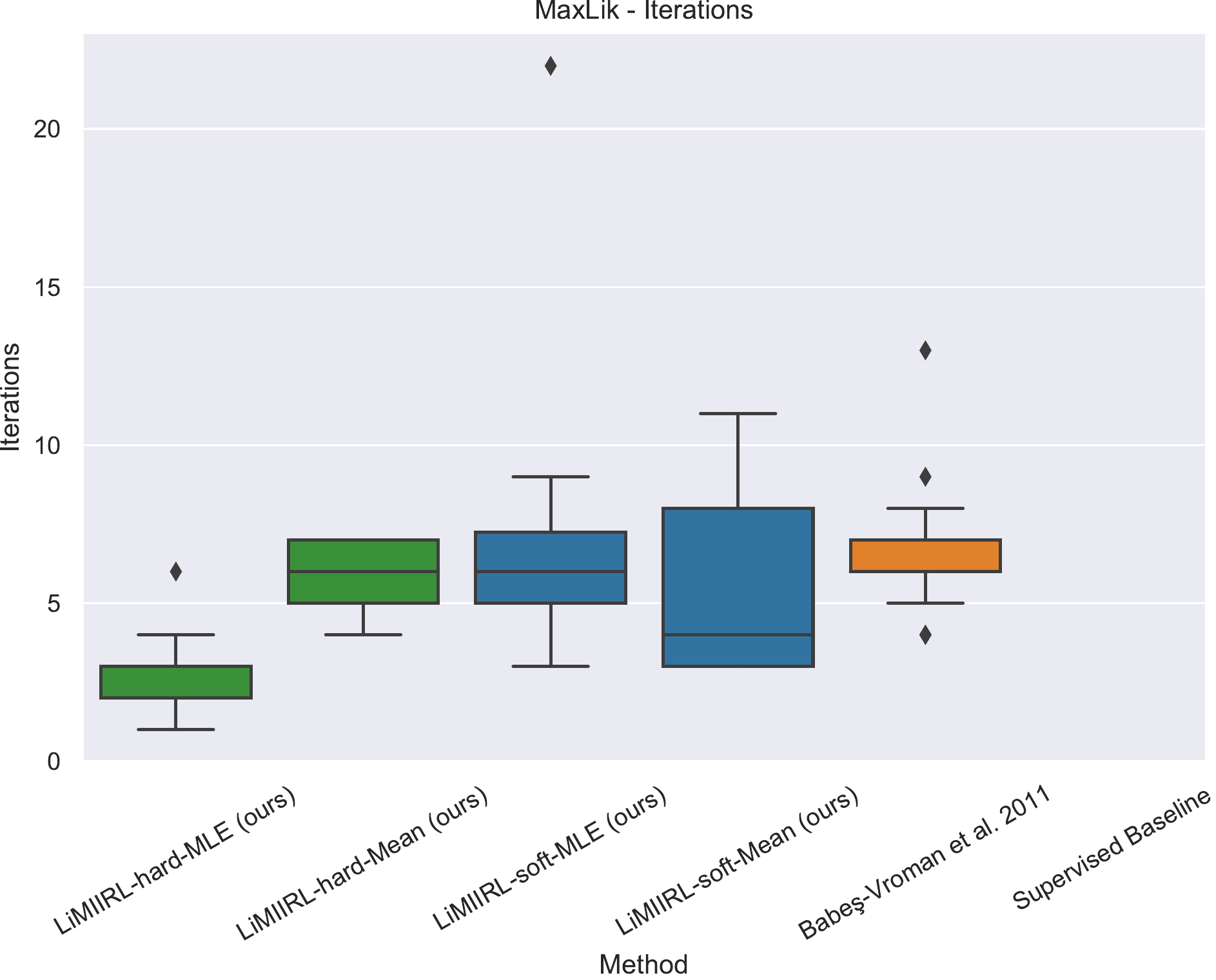}}
        \caption[]{%
            \textit{number of iterations}
        }
        \label{fig:ew-maxlik-iterations}
        \vspace{20pt}
    \end{subfigure}
    \begin{subfigure}[t]{0.48\linewidth}
        \centering
        \makebox[\textwidth][c]{\includegraphics[width=\linewidth]{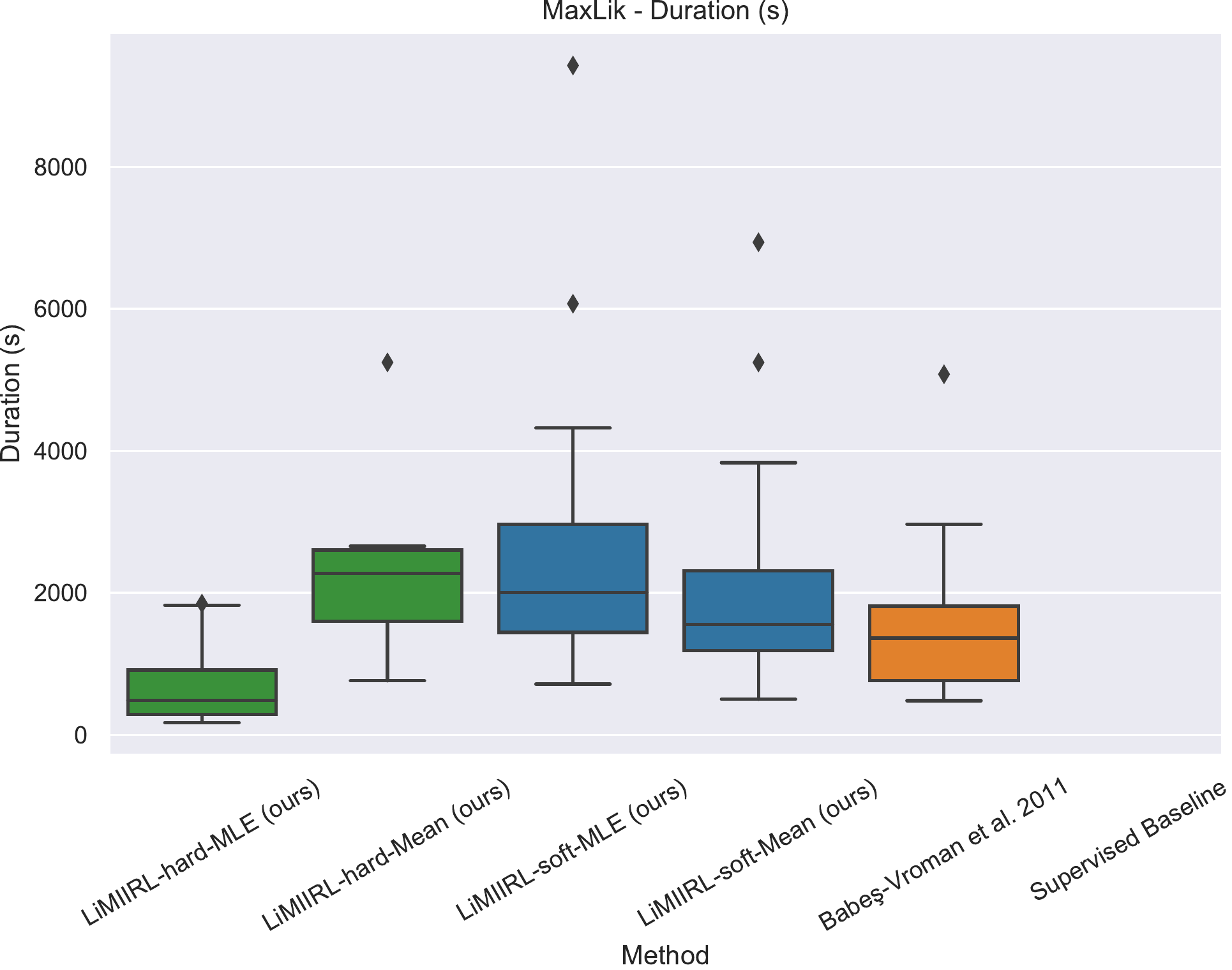}}
        \caption[]{%
            \textit{wall time in seconds}
        }
        \label{fig:ew-maxlik-duration}
    \end{subfigure}
    
    \begin{subfigure}[t]{0.48\linewidth}
        \centering
        \makebox[\textwidth][c]{\includegraphics[width=\linewidth]{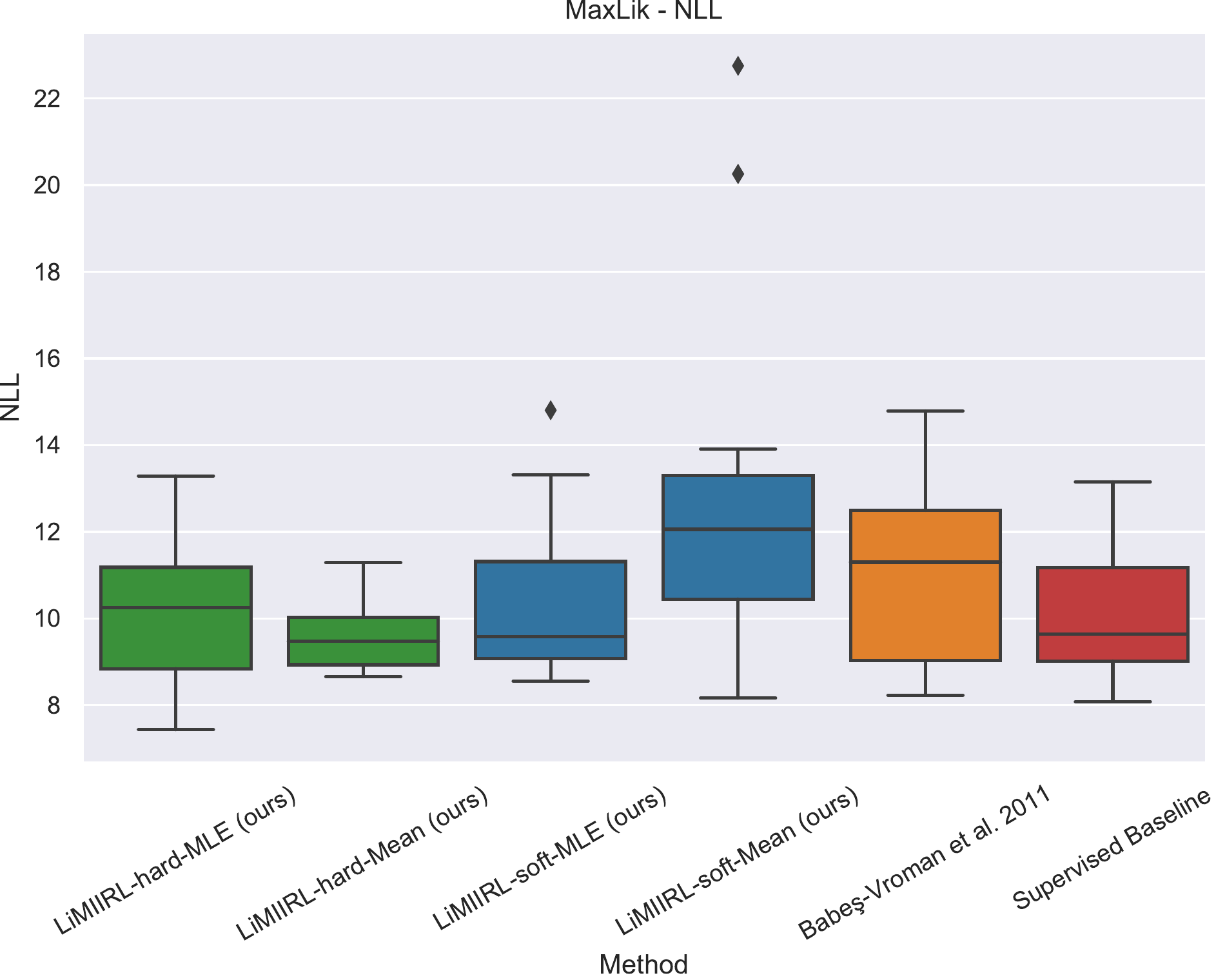}}
        \caption[]{%
            \textit{final mixture NLL}
        }
        \label{fig:ew-maxlik-nll}
        \vspace{20pt}
    \end{subfigure}
    \begin{subfigure}[t]{0.48\linewidth}
        \centering
        \makebox[\textwidth][c]{\includegraphics[width=\linewidth]{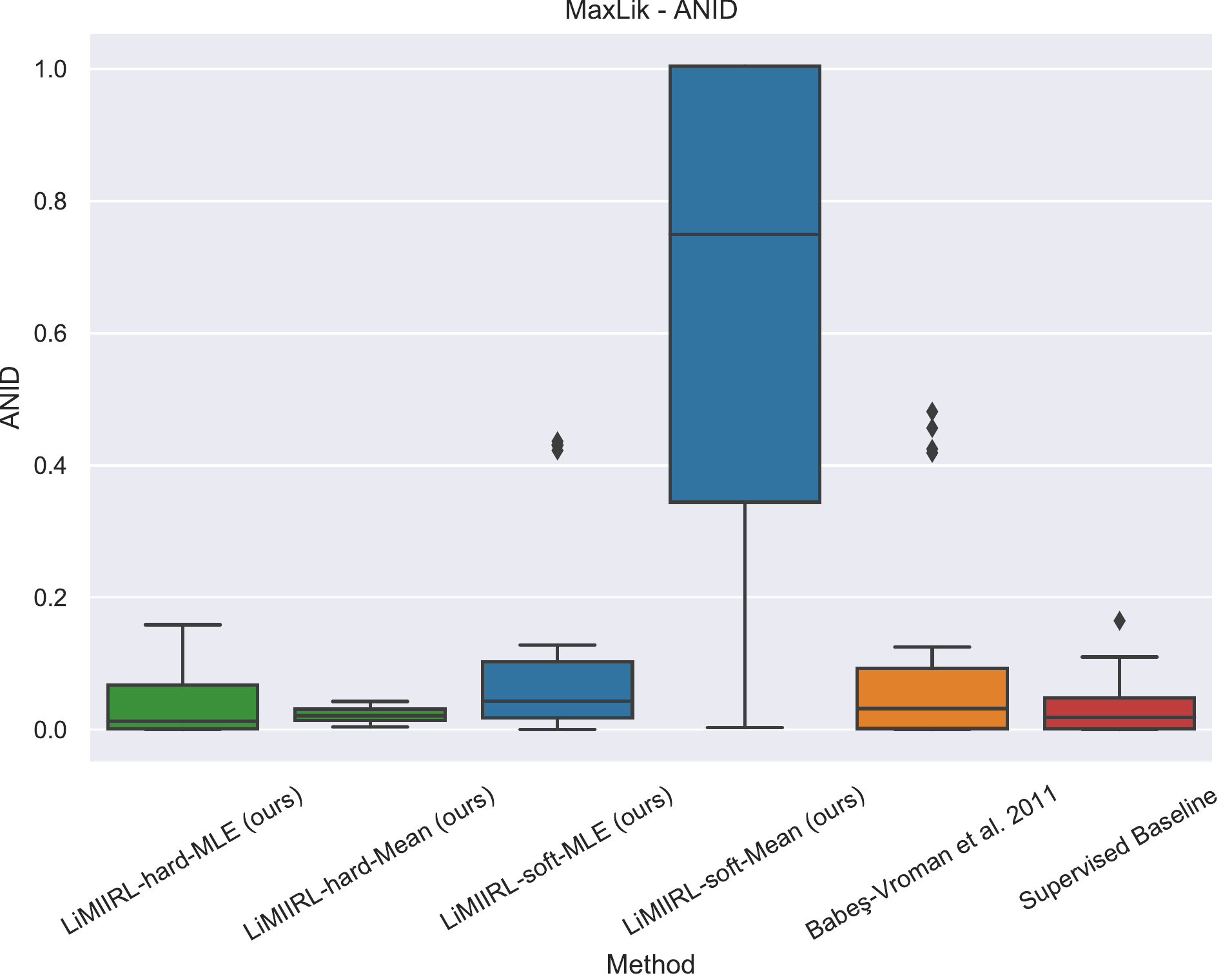}}
        \caption[]{%
            \textit{clustering performance (ANID)}
        }
        \label{fig:ew-maxlik-anid}
    \end{subfigure}
    
    \begin{subfigure}[t]{0.48\linewidth}
        \centering
        \makebox[\textwidth][c]{\includegraphics[width=0.9\linewidth]{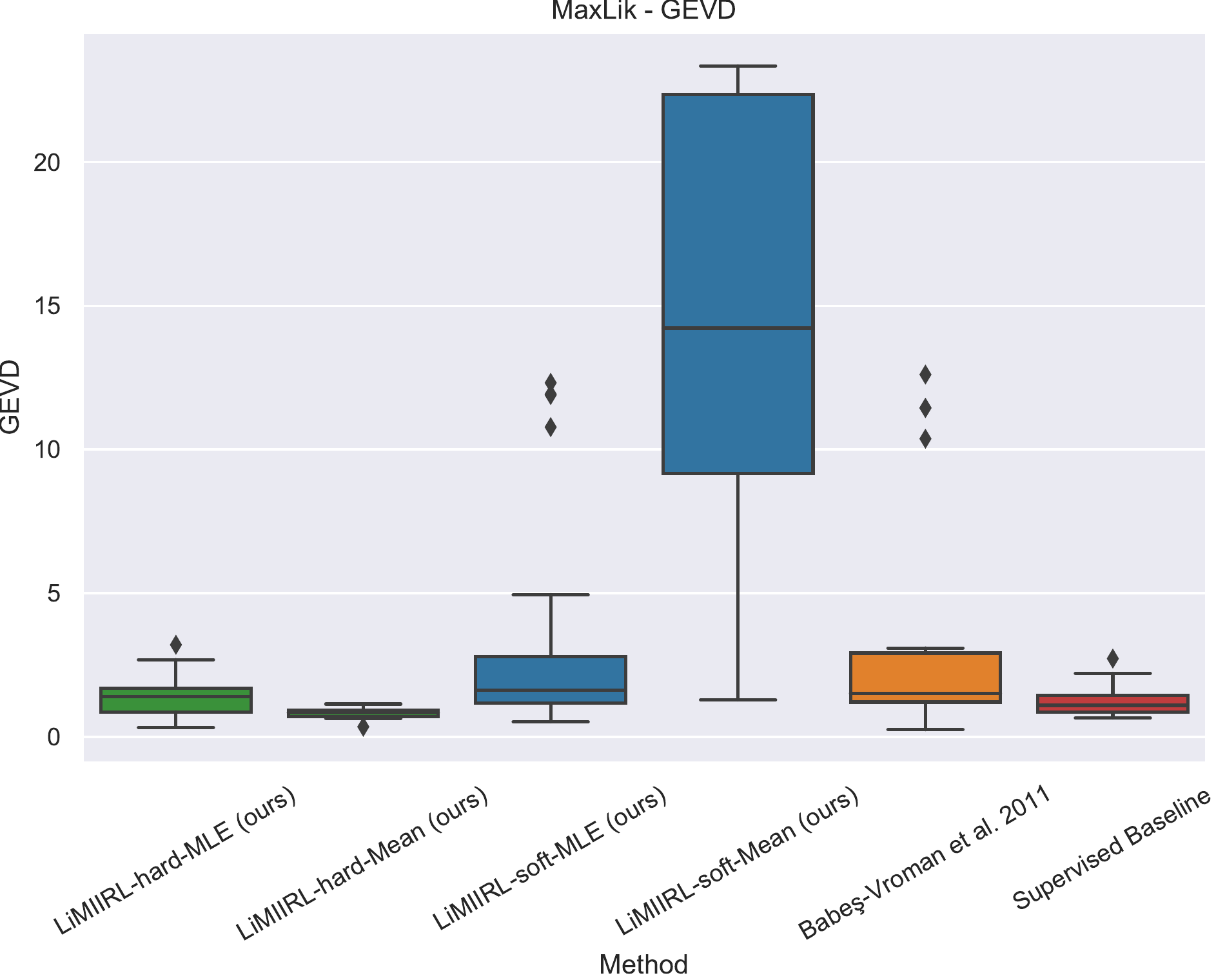}}
        \caption[]{%
            \textit{reward ensemble performance (GEVD)}
        }
        \label{fig:ew-maxlik-gevd}
    \end{subfigure}
    \caption{%
        \textit{ElementWorld} results with multi-intent ML-IRL.
        The label `Babe\c{s}-Vroman et al. 2011' indicates EM with random initialisation.
        Similarly to MaxEnt IRL (above), the hard clustering strategy consistently performs better than other initialisation strategies.
    }
    \label{fig:ew-maxlik}
\end{figure}

\clearpage
\thispagestyle{empty}

\begin{figure}[p]
    \begin{subfigure}[t]{0.48\linewidth}
        \centering
        \makebox[\textwidth][c]{\includegraphics[width=\linewidth]{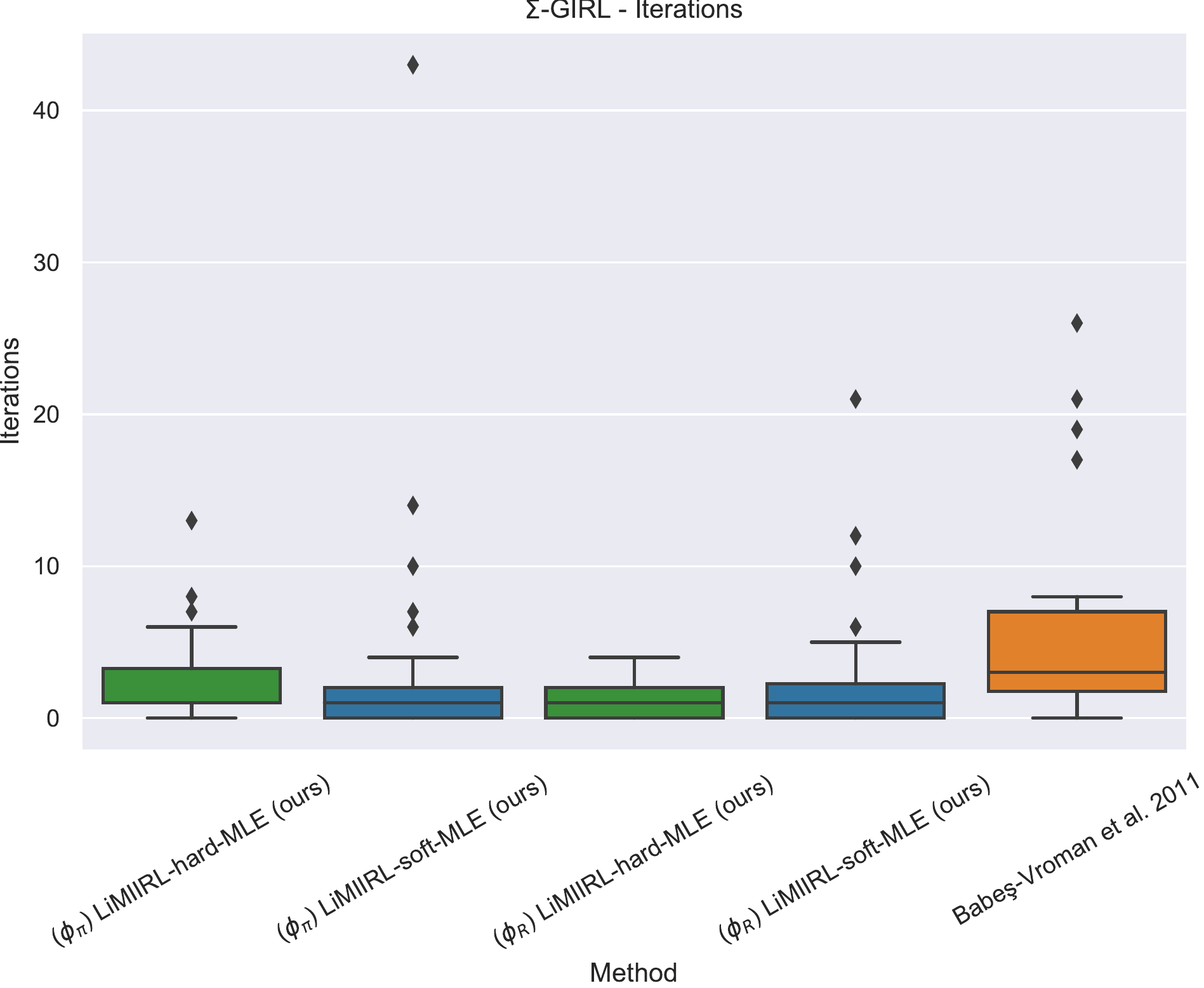}}
        \caption[]{%
            \textit{number of iterations.}
        }
        \label{fig:ew-sgirl-iterations}
        \vspace{20pt}
    \end{subfigure}
    \begin{subfigure}[t]{0.48\linewidth}
        \centering
        \makebox[\textwidth][c]{\includegraphics[width=\linewidth]{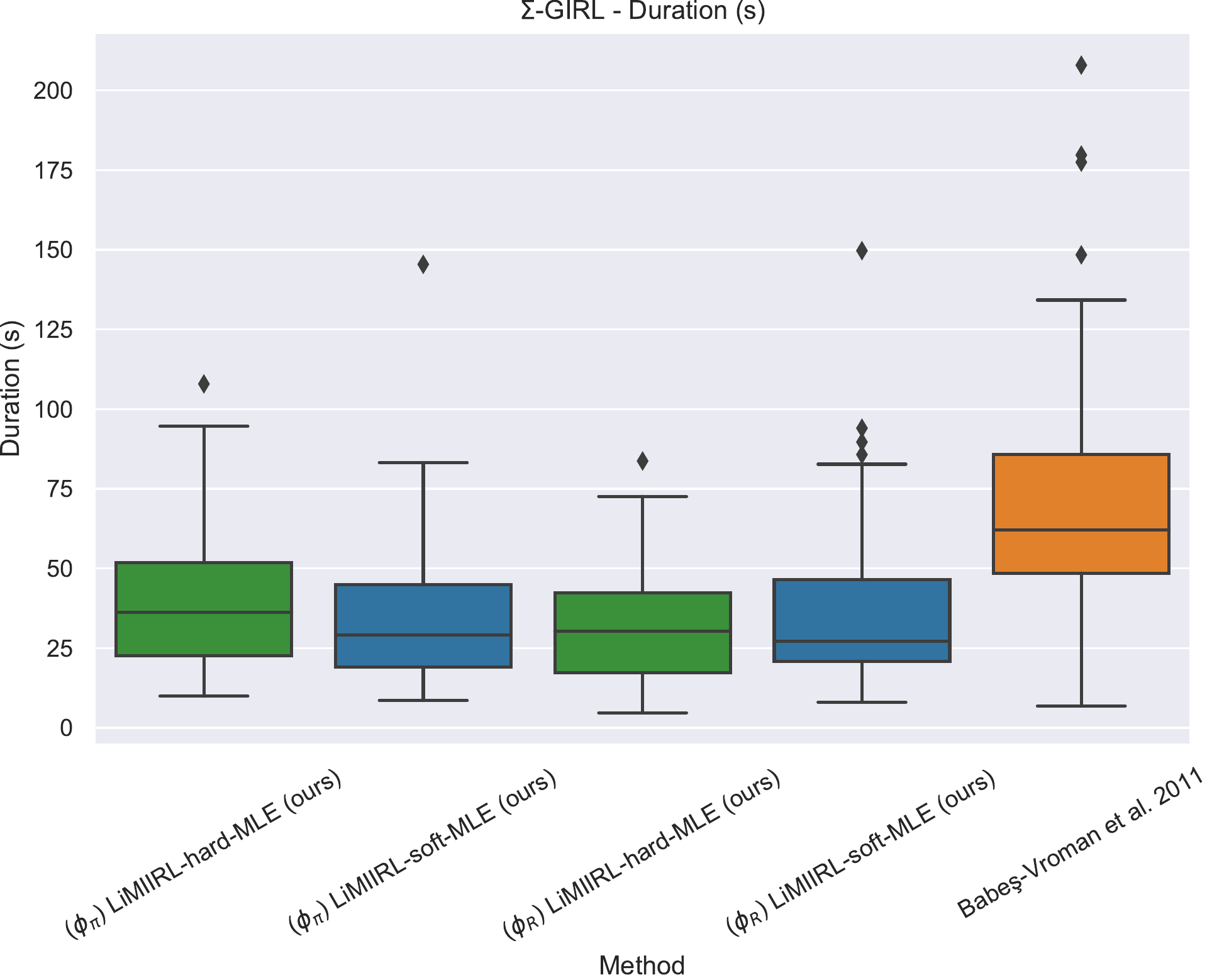}}
        \caption[]{%
            \textit{wall time in seconds.}
        }
        \label{fig:ew-sgirl-duration}
    \end{subfigure}
    
    \begin{subfigure}[t]{0.48\linewidth}
        \centering
        \makebox[\textwidth][c]{\includegraphics[width=\linewidth]{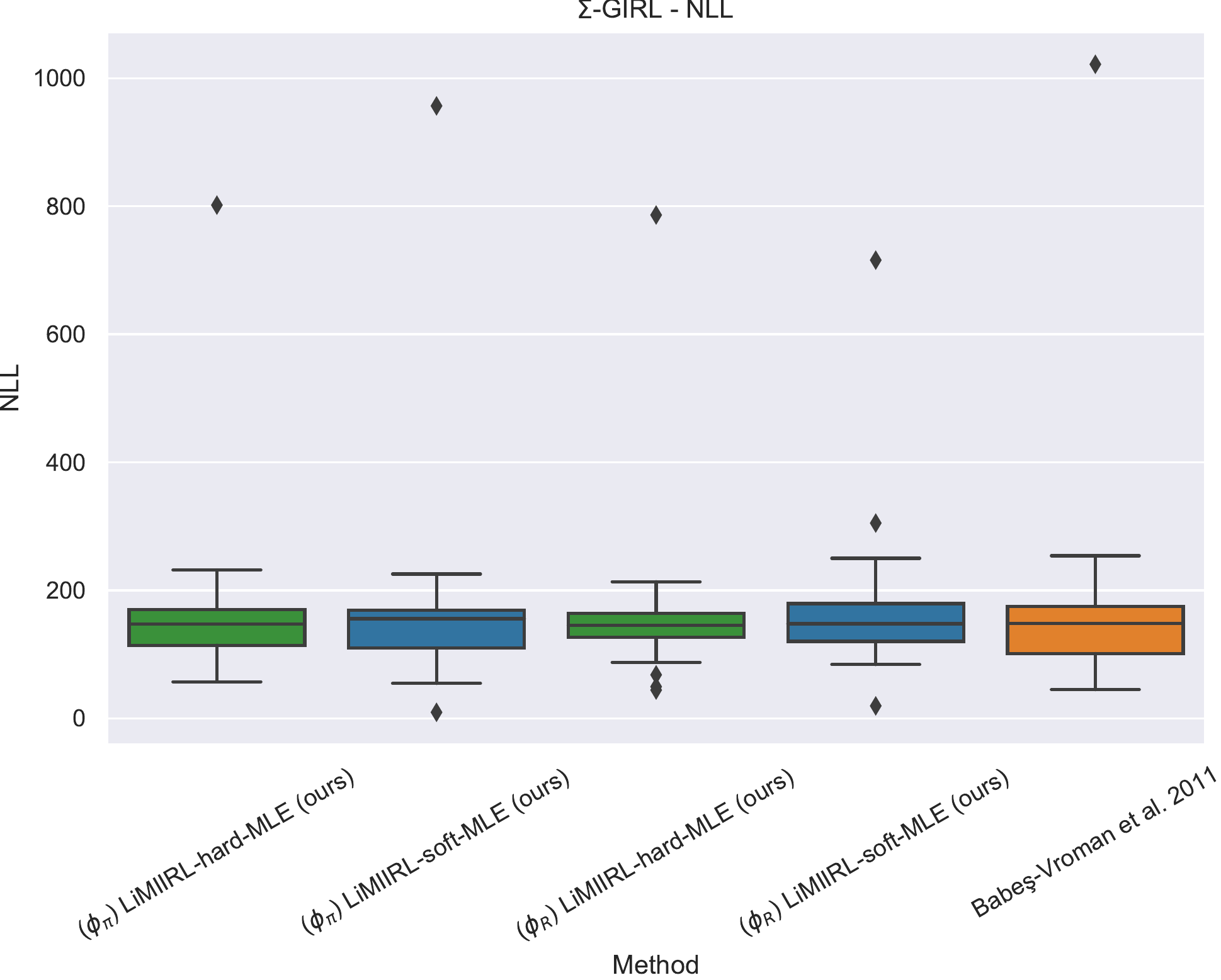}}
        \caption[]{%
            \textit{final mixture NLL}
        }
        \label{fig:ew-sgirl-nll}
        \vspace{20pt}
    \end{subfigure}
    \begin{subfigure}[t]{0.48\linewidth}
        \centering
        \makebox[\textwidth][c]{\includegraphics[width=\linewidth]{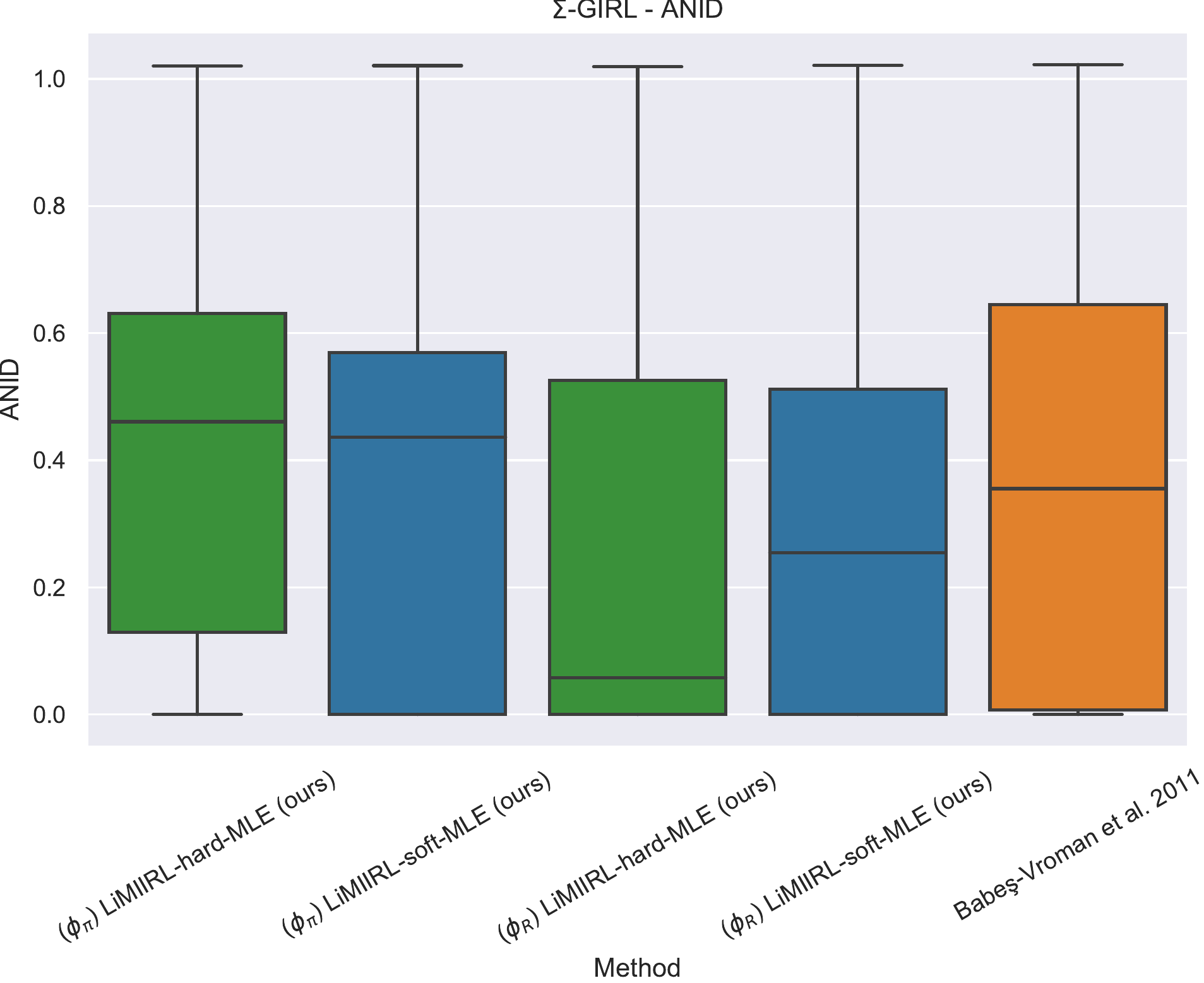}}
        \caption[]{%
            \textit{clustering performance (ANID)}
        }
        \label{fig:ew-sgirl-anid}
    \end{subfigure}
    \caption{%
        \textit{PuddleWorld} results with semi-supervised multi-intent $\Sigma$-GIRL (see explanation in \cref{subsec:lim-results}).
        The label `Babe\c{s}-Vroman et al. 2011' indicates EM with random initialisation.
        The $\Sigma$-GIRL algorithm differentiates between the reward feature function $\phi_R$ and the policy feature function $\phi_\pi$ in this environment -- we report results using both as the choice of feature vector for LiMIIRL initialisation, and find that they are roughly equivalent.
        For this problem, we find LiMIIRL leads to ensembles with equivalent ANID and NLL scores and provides a small but not statistically significant improvement on number of iterations and duration to convergence.
        We hypothesize that LiMIIRL is less helpful in these experiments (compared to the fully-unsupervised MI-IRL \textit{ElementWorld} experiments) due to one of two factors: (a) the simplicity of the continuous \textit{PuddleWorld} problem (i.e. it only has two behaviour intents), and/or (b) the fact that the $\Sigma$-GIRL algorithm has already exploited the labels mapping trajectories to agents, reducing the variance present in the resultant EM clustering problem.
    }
    \label{fig:ew-sgirl}
\end{figure}

\clearpage
\section{LiMIIRL Sensitivity Analysis results}

Below we include a series of figures plotting the results of our extensive sensitivity analysis.
We briefly describe the procedure used for each experiment, and any interesting observations.

\paragraph{Number of demonstrations.}
We varied the total number of demonstrations with $N \in \{20, 40, 100, 200, 400\}$ to investigate the performance of our method in the small- and large-data regimes.

\begin{figure}[h]
    \centering
    \includegraphics[width=\linewidth]{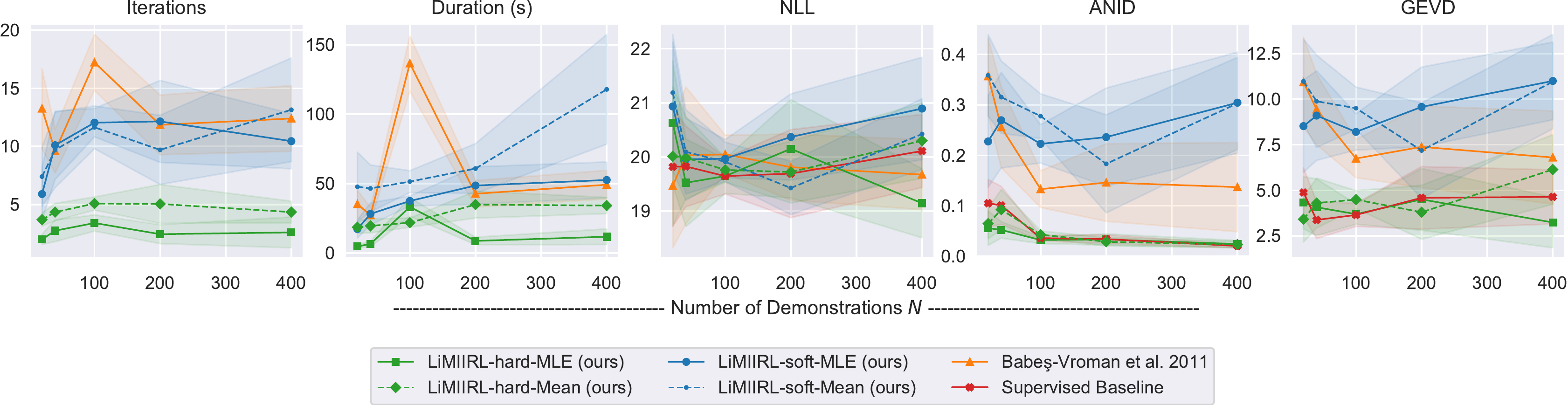}
    \caption[]{%
        Effect of varying number number of demonstrations $N$.
    }
    \label{fig:ew-sweeps:num-demos}
\end{figure}

We observed that the number of iterations stays fairly constant, whereas the wall time duration increases, suggesting that the M-Step (IRL reward solving), not the E-Step (updating cluster allocations) is the more burdensome part of the MI algorithm.
LiMIIRL with hard clustering generally outperforms the other initialisation strategies.

\paragraph{Cluster separation.}
We tested the algorithm with deterministic dynamics as well as with wind factor $w \in \{0.05, 0.1, 0.15, 0.2\}$, where larger wind values correspond to the behaviour intents being less separated in feature space (progressive violation of Assumption A).

In \cref{fig:ew-assumption-1}, we plot the mean of the inter-cluster margin $\E_{\{\Tau, \Tau'\}}[\min_{\tau \sim \Tau, \tau' \sim \Tau'} \|\phi(\tau) - \phi(\tau')\|]$ (where the expectation indicates all pairwise combinations of elements) with $E=3$ elements as the wind factor is increased.
This confirms that increasing the stochasticity of the transition dynamics leads to progressive violation of the cluster separation assumption.

\begin{figure}[h]
    \centering
    \makebox[\textwidth][c]{\includegraphics[width=0.6\linewidth]{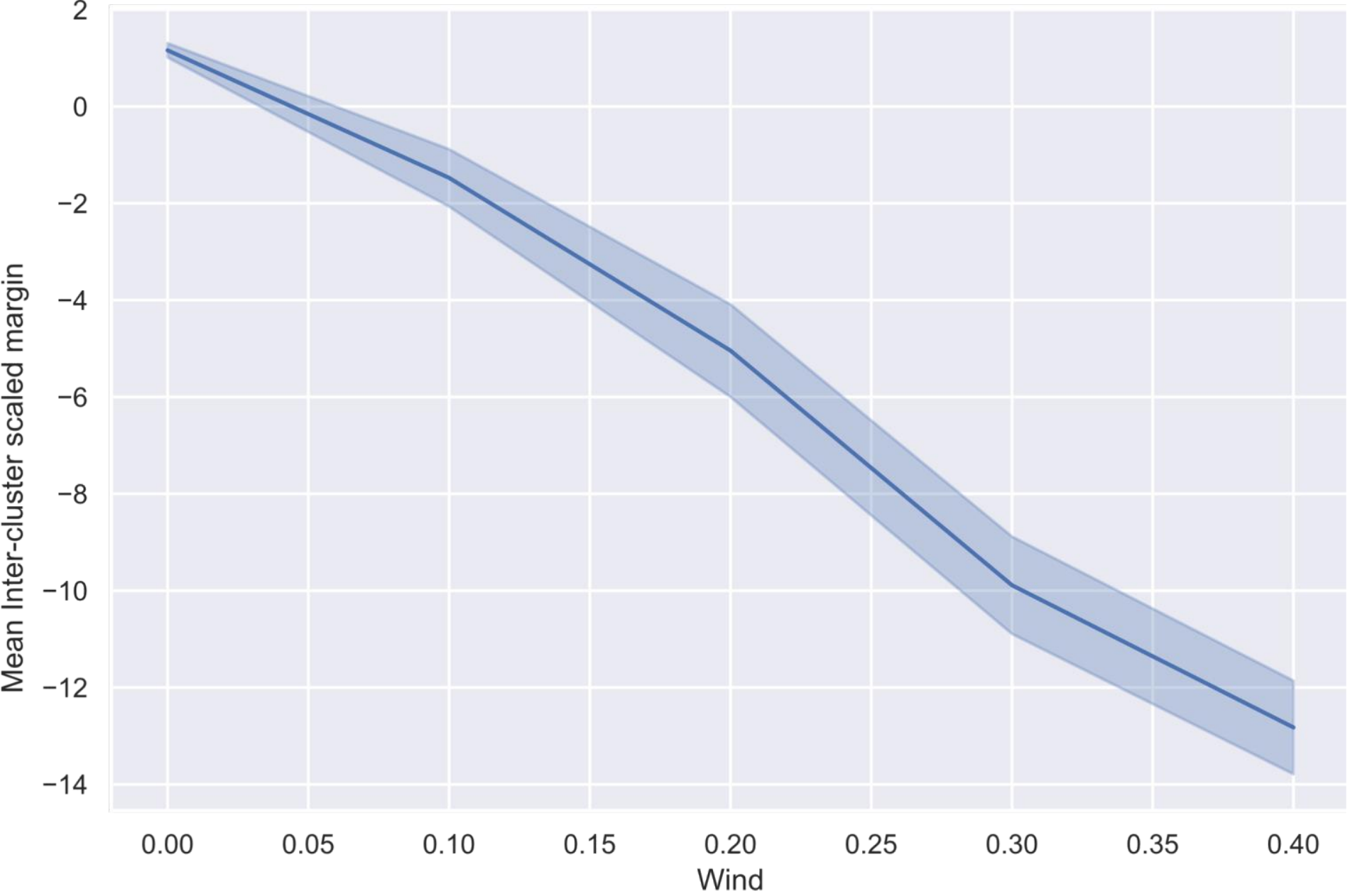}}
    \caption[]{%
        Assumption 1 is progressively violated as the \textit{ElementWorld} wind factor is increased.
        The plot shows the mean inter-cluster scaled margin with the 95\% confidence interval over 10 repeat experiments.
    }
    \label{fig:ew-assumption-1}
\end{figure}

\begin{figure}[h]
    \centering
    \makebox[\textwidth][c]{\includegraphics[width=\linewidth]{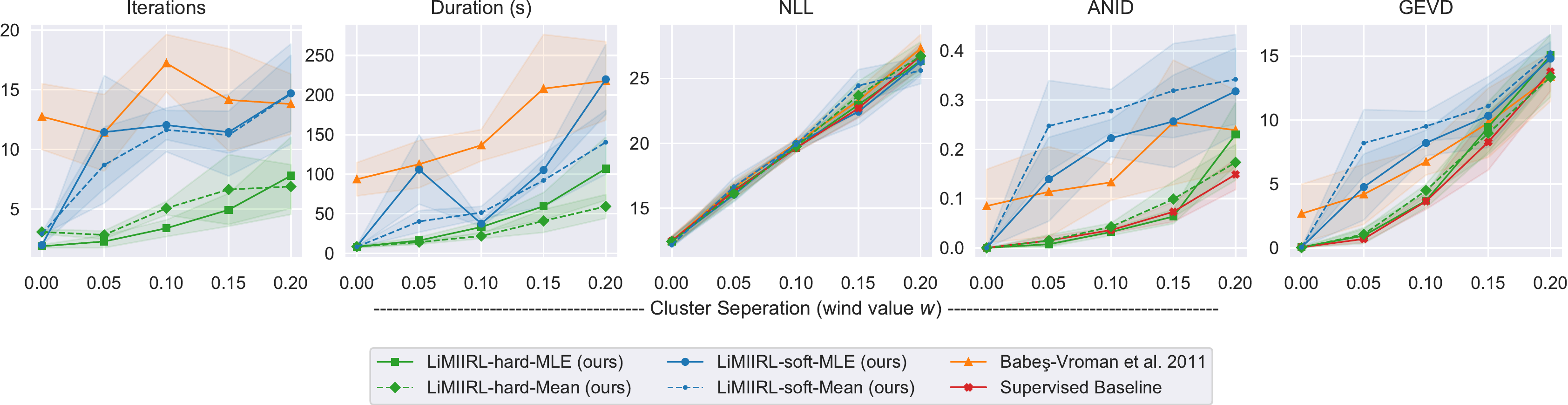}}
    \caption[]{%
        Effect of varying separation of the behaviour clusters (achieved by adjusting the `wind' factor $w$ in the dynamics)
    }
    \label{fig:ew-sweeps:cluster-separation}
\end{figure}

We observed that increasing the wind factor (violating Assumption 1 more) reduces the performance of all methods on all metrics, however in general, LiMIIRL initialisation with hard clustering still outperforms the other method.

\paragraph{Number of clusters.}
We fixed $E = 3$ and varied the number of learned clusters from $K = 1$ to $5$ to see how LiMIIRL performs when the number of learned clusters is mis-specified or unknown (\ie{} under- or over-clustering).

\begin{figure}[h]
    \centering
    \makebox[\textwidth][c]{\includegraphics[width=\linewidth]{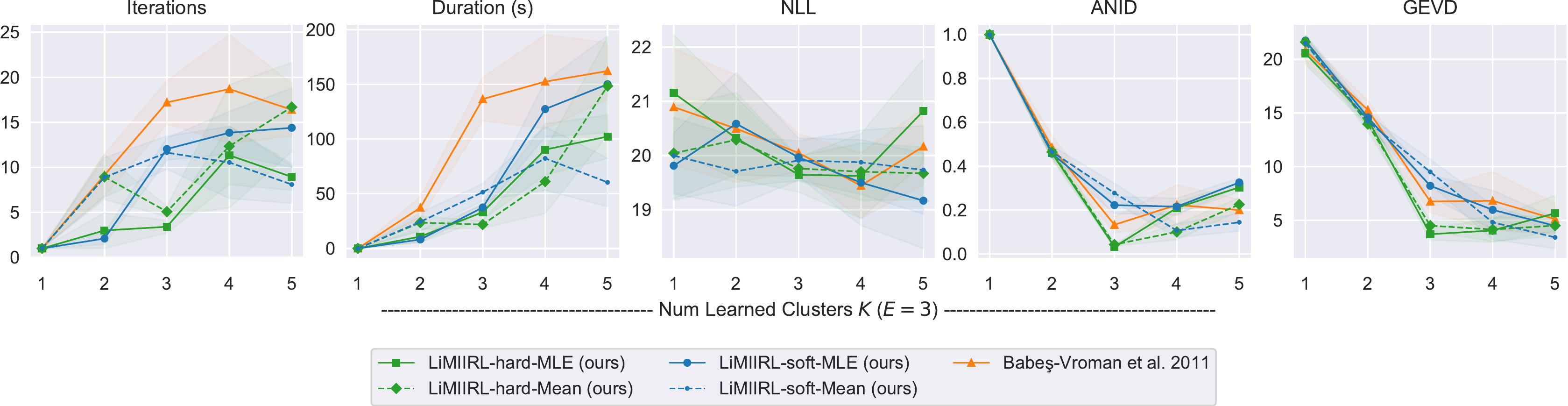}}
    \caption[]{%
        Effect of varying numbers of clusters $K$ (we fix $E=3$).
    }
    \label{fig:ew-sweeps:num-clusters}
\end{figure}

We observe that our LiMIIRL initialisation strategy appears most helpful when the number of clusters is under, or correctly specified for the environment.
We note that while the mixture NLL generally increases proportional to $K$, the ANID and GEVD metrics reach a minima / asymptote respectively at $K=3$ -- this confirms that these measures are suitable for evaluating reward ensemble performance, even for the case when the number of learned clusters is mis-specified.

We also note that in practice, the value of $K$ can be estimated using heuristics such as the Bayesian Information Criteria, using a non-parametric clustering as a pre-processing step, or using cross-fold validation.

\clearpage
\paragraph{Number of elements.}
We generated ElementWorld configurations with $K = E \in \{2, 3, 4, 5, 6\}$ to test the performance of our method as the number of behaviour intents increases.

\begin{figure}[h!]
    \centering
    \makebox[\textwidth][c]{\includegraphics[width=\linewidth]{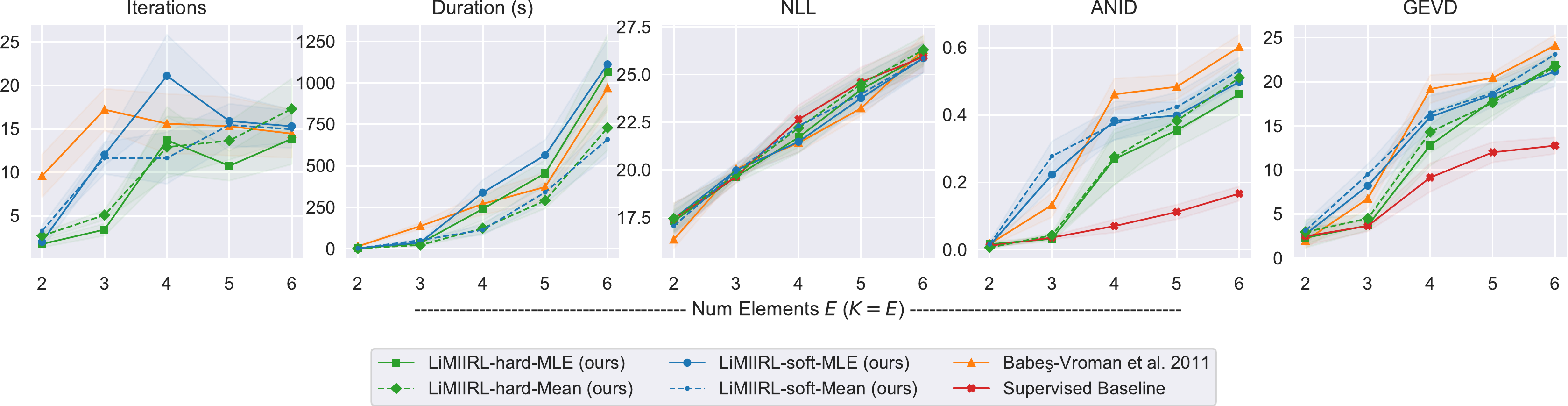}}
    \caption[]{%
        Effect of varying number number of ground truth elements $E$ (we fix $K = E$)
    }
    \label{fig:ew-sweeps:num-elements}
\end{figure}

We observe that, as expected, the performance of all metrics (except the supervised baseline) gets worse as the number of behaviour intents increases.
LiMIIRL with hard clustering appears to provide even more benefit over other initialisation strategies, especially for large numbers of behaviour intents.

\paragraph{Cluster imbalance.}
We tested the performance of our method when the behaviour intentions are not uniformly present in the demonstration dataset.
Instead of sampling demonstration policies uniformly, we selected demonstration policies from the ground truth ensemble according to a normalized geometric distribution $\rho_{k} \propto p (1 - p)^{k}$, varying the parameter $p \in \{0.0, 0.1, 0.2, 0.3, 0.4\}$.

\begin{figure}[h!]
    \centering
    \makebox[\textwidth][c]{\includegraphics[width=\linewidth]{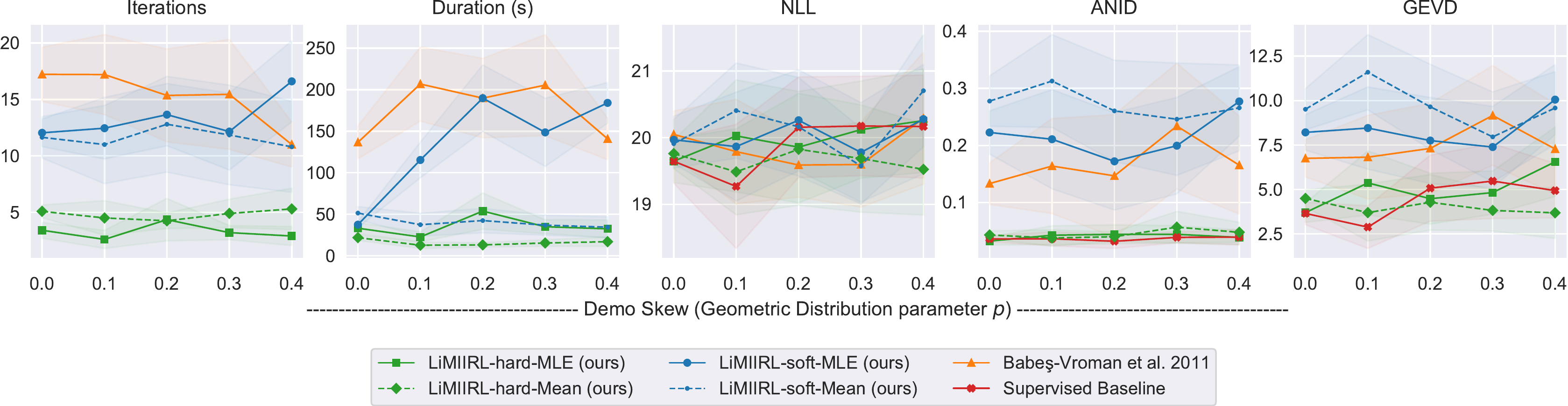}}
    \caption[]{%
        Effect of varying a varying degree of demonstration data cluster imbalance (achieved by sampling demonstrations from the ground truth ensemble under a geometric distribution with parameter $p$).
    }
    \label{fig:ew-sweeps:cluster-imbalance}
\end{figure}

We observe that LiMIIRL with hard clustering appears once again to outperform the other initialisation strategies, even in the presence of heavily skewed cluster sizes.
Interestingly, the soft initialisation (based on a Gaussian Mixture Model) appears particularly sensitive to cluster imbalance, suggesting this method should be avoided when the true cluster distribution is unknown.

\clearpage
\section{Driver Forecasting Experiment: Learned Reward Ensembles}

Below we include plots that visualise the best learned reward ensemble for each of the algorithms LiMIIRL-hard, LiMIIRL-soft, and for the EM algorithm with random initialization \cite{Babes-Vroman2011}.
Each figure is divided into sections corresponding to the specific feature - e.g. road type, or geographical region.
Within each section, the x-axis shows the feature dimensions, and the plot indicates the learned reward parameter for that feature dimension.
Each figure shows a plot for each reward in the learned ensemble, and the plot thickness is proportional to the ensemble component weight $\rho_k$.

\begin{figure}[p]
    \vspace{-40pt}
    \centering
    \begin{subfigure}{0.7\linewidth}
        \centering
        \includegraphics[width=\linewidth]{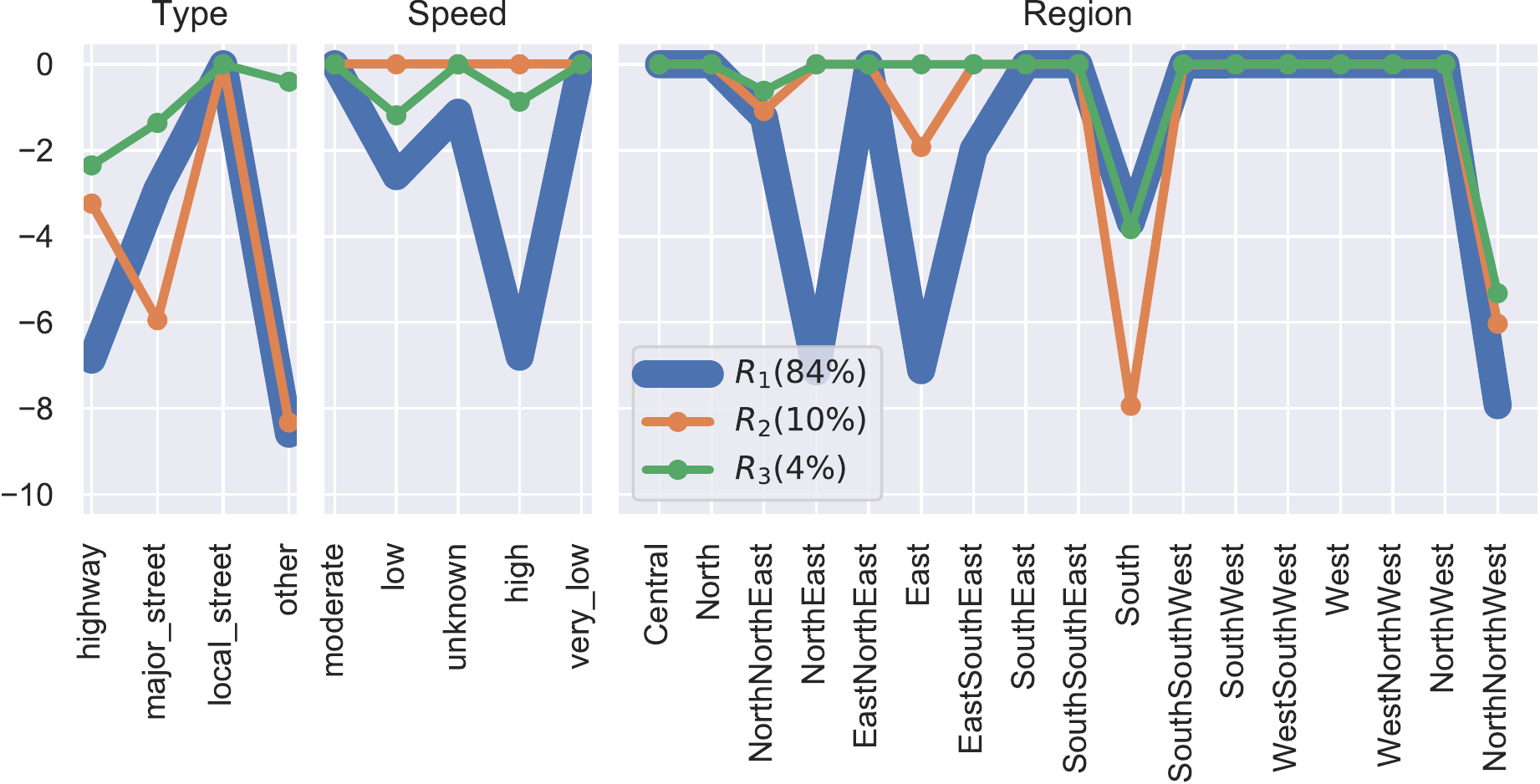}
        \caption[]{%
            Best learned reward ensemble for LiMIIRL-hard.
        }
    \end{subfigure}%
    \vspace{5pt}
    
    \begin{subfigure}{0.7\linewidth}
        \centering
        \includegraphics[width=\linewidth]{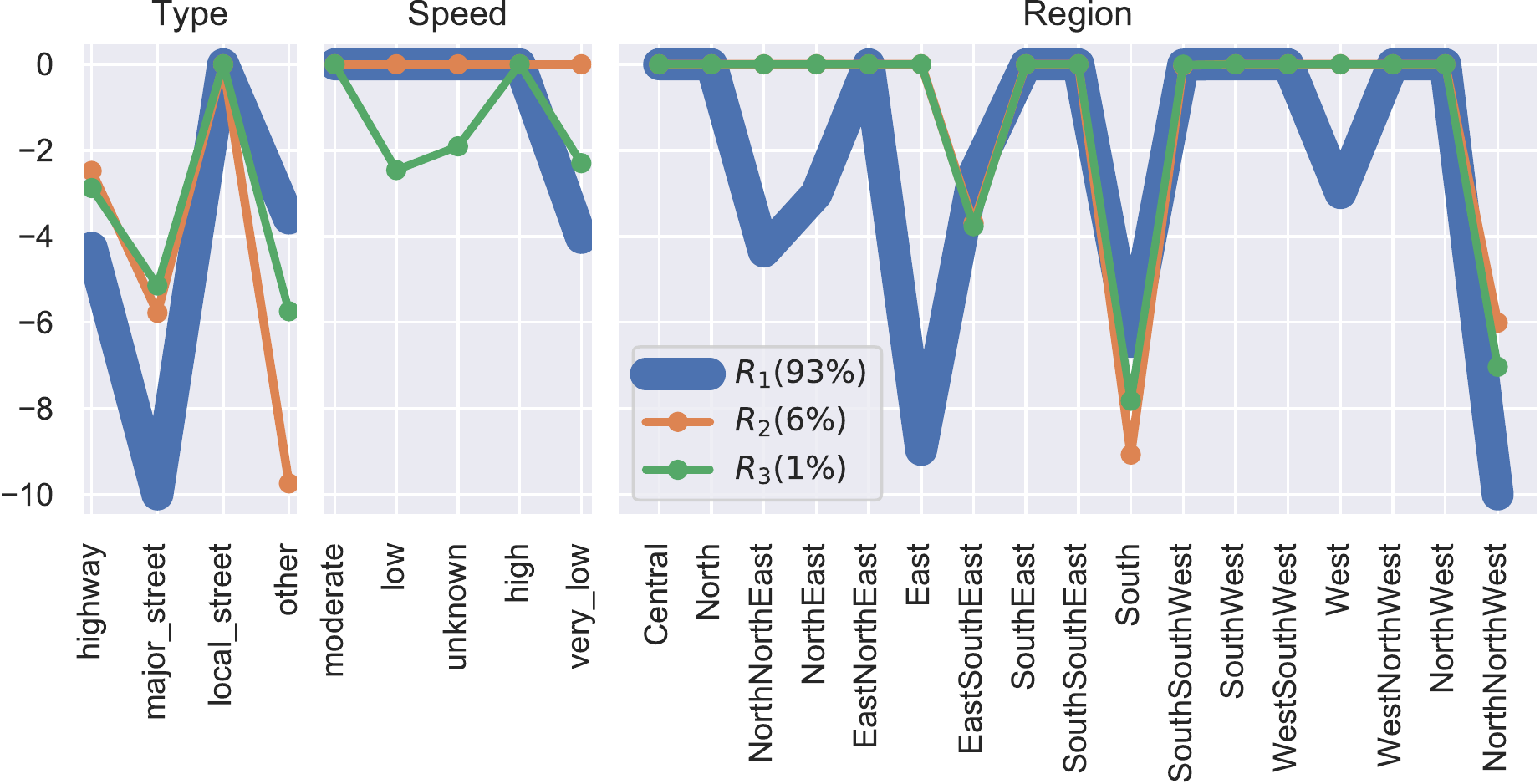}
        \caption[]{%
            Best learned reward ensemble for LiMIIRL-soft.
        }
    \end{subfigure}%
    \vspace{5pt}
    
    \begin{subfigure}{0.7\linewidth}
        \centering
        \includegraphics[width=\linewidth]{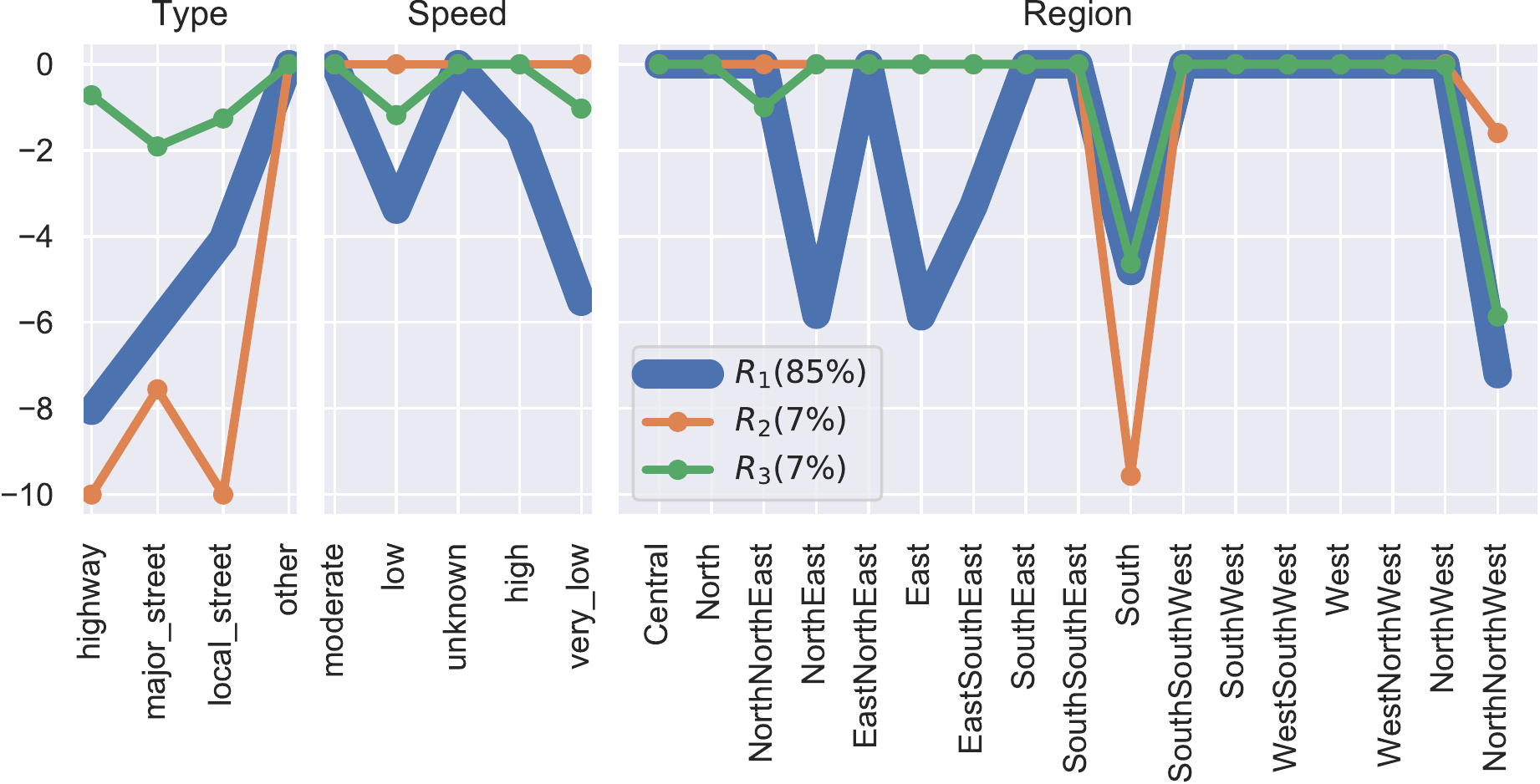}
        \caption[]{%
            Best learned reward ensemble for \citet{Babes-Vroman2011}.
        }
    \end{subfigure}%
    \caption[]{%
        Learned reward ensembles for the \textit{Porto} driver behaviour forecasting experiment.
    }
    \label{fig:ew-sweeps}
\end{figure}

\end{document}